\documentclass{article}
\usepackage{iclr2027_conference,times}
\iclrfinalcopy
\usepackage[utf8]{inputenc}
\usepackage[T1]{fontenc}
\usepackage{microtype}
\usepackage{amsmath,amssymb,amsthm,mathtools}
\usepackage{booktabs,multirow}
\usepackage{graphicx}
\usepackage{xcolor}
\usepackage{hyperref}
\usepackage{url}
\usepackage{enumitem}
\usepackage{tikz}
\usepackage{wrapfig}
\usepackage{needspace}
\usepackage{placeins}
\usepackage{adjustbox}
\usepackage{subcaption}
\usepackage{capt-of}
\usepackage{float}

\usetikzlibrary{arrows.meta,backgrounds,calc,fit,positioning,shapes.geometric}

\newcommand{\method}{\textsc{CADT}}
\newcommand{\methodfull}{Concept Attribution through Dynamic Trajectories}

\newcommand{\best}[1]{\textbf{#1}}

\newtheorem{theorem}{Theorem}
\newtheorem{proposition}{Proposition}

\theoremstyle{definition}

\title{Attributing HOW, Not Just WHICH: Counterfactual Response Trajectories for Diffusion Models}

\author{%
  \textbf{Haoqian Zhang\textsuperscript{1}}\quad
  \textbf{Ziyuan Yang\textsuperscript{1}}\quad
  \textbf{Zerui Shao\textsuperscript{1}}\quad
  \textbf{Yi Zhang\textsuperscript{1}}\\
  \normalfont\textsuperscript{1}Sichuan University
}

\begin{document}
\maketitle

\begin{abstract}
Diffusion models have achieved remarkable success in image generation, yet tracing their outputs to individual training examples remains challenging. Existing attribution methods often compress factor-specific effects into scalar responses, making distinct internal changes indistinguishable. This is particularly limiting for diffusion models, where semantic factors emerge through evolving representation dynamics during denoising. We therefore reformulate diffusion data attribution as attributing factor-induced internal response trajectories. In this paper, we propose a novel Concept Attribution method through Dynamic Trajectories(CADT). 
We argue that attribution should therefore ask not only \emph{which} examples matter, but also \emph{how} their influence unfolds during generation.
Specifically, we construct matched counterfactual pairs at identical noisy states to isolate factor-specific representation displacements, and model their directional and magnitude evolution across denoising as dynamic attribution signatures. 
%These stage-wise signatures are integrated across denoising to form trajectory descriptors for both training examples and generated queries. CADT then calibrates these descriptors against the variations observed across the training bank and uses the resulting covariance-aware positive-semidefinite kernel $K_\lambda$ to compare query and training trajectories, producing the final attribution scores. 
For each training example and generated query, CADT extracts stage-wise feature vectors and integrates them along the denoising process to form a trajectory descriptor. Applying the same construction across the training set yields a bank of factor-specific trajectory descriptors. The covariance statistics of this bank are then used to construct a covariance-aware positive-semidefinite kernel. CADT uses this kernel to calibrate the query and training representations, and compares the calibrated query trajectory with each training trajectory to produce the final training-sample attribution scores.
Experiments on multiple public datasets show consistent improvements over existing diffusion attribution baselines across hierarchical, compositional, and style attribution.

\end{abstract}

\section{Introduction}

Diffusion models have emerged as a powerful class of generative models, achieving remarkable progress in high-quality image synthesis and enabling flexible control over generated content \citep{ho2020ddpm,nichol2021improved}. As their behavior is shaped by large training collections, understanding which training examples contribute to a generated output is important for interpreting model behavior, tracing data provenance, and diagnosing memorization~\citep{somepalli2023replication,carlini2023extracting}. This motivates \emph{data attribution} for diffusion models, which aims to trace a generated output back to the training examples that contributed to its formation.

Recent diffusion attribution methods extend influence- and gradient-based attribution to generative models through diffusion-specific objectives, and sampling procedures \citep{dai2023training,georgiev2023journey,park2023trak,
zheng2024dtrak,xie2024diffusionretrac,mlodozeniec2024influence,
lin2025das}. Most, however, ultimately summarize influence through a scalar response associated with the generated image. This is restrictive for factor-specific attribution. Distinct internal responses can collapse to similar scalar values, even when they emerge at different denoising stages or evolve in different directions. Moreover, image-level responses can mix multiple semantic factors. 
% For instance, the object identity and visual style of a watercolor dog may be supported by different subsets of training examples, yet a single attribution response can conflate their effects.

% Concept-TRAK \citep{park2025concept} extends diffusion attribution toward concept-specific analysis, but still represents concept influence through a predefined scalar utility. This scalar view can obscure the evolving internal responses through which semantic factors emerge over the denoising process \citep{luo2023hyperfeatures,tang2023dift}.

We therefore argue that factor-specific attribution should characterize not only \emph{how much} a training example affects a target factor, but also \emph{how the corresponding internal response evolves during generation}. Two training examples can receive similar scalar attribution scores while supporting the target factor through qualitatively different generative processes: one may contribute primarily during early denoising, another during later refinement, and their internal responses may evolve along different representation directions. More importantly, these differences can correspond to different semantic roles. For a generated watercolor dog, for example, one training image may primarily support object identity whereas another supports artistic style, even if their aggregated scalar responses are similar. Collapsing such trajectories into a single value therefore obscures not only when influence arises, but also what aspect of the generated content that influence supports.
% , as further illustrated in Appendix~\ref{app:concept-scalar}.

% We therefore argue that factor-specific attribution should characterize not only \emph{how much} a training example affects a target factor, but also \emph{how the corresponding internal response evolves during generation}. Response trajectories retain information discarded by scalar aggregation: a factor may emerge at different denoising stages, exhibit different response strengths, or evolve along different representation directions. Indeed, we observe that training examples with similar scalar responses can exhibit substantially different internal denoising dynamics, as further illustrated in Appendix~\ref{app:concept-scalar}, where similar scalar responses correspond to distinct temporal contribution patterns and internal response trajectories.

This motivates two requirements for factor-specific diffusion attribution. First, attribution should preserve how the relevant response evolves throughout denoising rather than collapsing it into a single scalar. Second, this response should isolate the target factor from other semantic variations entangled in the same internal representations. We introduce \emph{\methodfull{} (\method)} to jointly satisfy these requirements. At each selected denoising stage, \method{} evaluates the target condition and a matched reference condition at the \emph{same noisy latent}, and takes their representation difference as a counterfactual response. By holding the diffusion state fixed, this comparison suppresses variations shared by the two conditions and isolates the response associated with the target factor. \method{} then traces how this factor-specific response evolves across denoising stages, capturing both its magnitude and directional evolution. Finally, the stage-wise responses are integrated into a trajectory descriptor that characterizes how the target factor progressively emerges during generation, for both training examples and generated queries.

Directly comparing these descriptors, however, can still be dominated by variations that are broadly shared across the training set and are therefore less informative for attribution. \method{} consequently uses the training descriptor bank to adapt the comparison geometry itself. Specifically, it constructs a covariance-aware positive-semidefinite kernel $K_\lambda$ that accounts for shared trajectory variation and emphasizes more discriminative directions. Query and training descriptors are compared under this bank-adaptive metric, and the resulting similarities are converted into attribution scores that identify the training examples most closely associated with the factor-specific generative trajectory. The resulting kernel further admits both a ridge-reconstruction interpretation and a representation-level leave-one-descriptor-out interpretation. We evaluate our proposed \method{} on multiple public datasets, covering compositional, hierarchical, corrupted-source, style, and object--style attribution settings. Across these settings, \method{} consistently improves factor-specific attribution over existing diffusion attribution baselines. Our contributions are summarized as follows:

\begin{itemize}[leftmargin=*]
\item We reformulate factor-specific diffusion attribution in terms of \emph{internal response trajectories}, preserving denoising dynamics that are lost under scalar attribution responses.

\item We construct matched counterfactual responses at identical noisy states to isolate factor-induced internal changes, and model their response magnitude and directional evolution across denoising as dynamic attribution signatures.

\item We integrate these dynamics into trajectory descriptors for training examples and generated queries, and compare them through a training-bank-adaptive covariance-aware kernel to obtain attribution scores across diverse factor-specific attribution settings.
\end{itemize}

\section{Related Work}

\paragraph{Training-data attribution.}
Training-data attribution studies how individual training examples affect model behavior. Classical approaches include leave-one-out retraining, influence functions, representer-based explanations and data valuation \citep{koh2017influence,ghorbani2019shapley,yeh2018representer}. More scalable methods such as TracIn, datamodels, and TRAK approximate training-example influence using optimization trajectories, subset models, or projected gradients \citep{pruthi2020tracin,ilyas2022datamodels,park2023trak}.

\paragraph{Attribution for diffusion models.}
Recent work extends training-data attribution to generative and diffusion models. Diffusion-specific methods estimate influence using denoising-loss gradients, projected features, sampling trajectories, or predictive-distribution changes \citep{dai2023training,georgiev2023journey,zheng2024dtrak,
xie2024diffusionretrac,mlodozeniec2024influence,lin2025das}. These approaches adapt scalable attribution to the iterative denoising structure of diffusion models.

\paragraph{Concept-level and localized explanations.}
Beyond whole-image attribution, recent work studies more fine-grained explanations of diffusion models. Concept-TRAK introduces concept-aware utility functions for concept-level training-data attribution \citep{park2025concept}, while DAAM analyzes token-level conditioning through cross-attention \citep{tang2023daam}. Other approaches investigate spatially localized training-data influence within generated images \citep{nguyen2025region}.

\section{Preliminaries}
\label{sec:preliminaries}

\subsection{Diffusion Models}
\label{sec:diffusion-model}

Diffusion models synthesize data by reversing a predefined noise-corruption process \citep{sohl2015nonequilibrium,ho2020ddpm,song2021score}. Given a clean sample $\mathbf{x}_0$, the forward process gradually adds Gaussian noise over $T$ timesteps. Let $\beta_t\in(0,1)$ denote the variance schedule, $\alpha_t=1-\beta_t$, and $\bar\alpha_t=\prod_{s=1}^{t}\alpha_s$. The noised state at timestep $t$ can be sampled directly as: 
\begin{equation}
\mathbf{x}_t
=
\sqrt{\bar\alpha_t}\mathbf{x}_0
+
\sqrt{1-\bar\alpha_t}\boldsymbol\epsilon,
\qquad
\boldsymbol\epsilon\sim\mathcal N(\mathbf 0,\mathbf I).
\end{equation}

A conditional denoiser $\boldsymbol\epsilon_\theta(\mathbf{x}_t,t,\mathbf c)$ predicts the injected noise under an optional condition $\mathbf c$. In the common noise-prediction parameterization, the model minimizes: 
\begin{equation}
\mathcal L_{\mathrm{simple}}
=
\mathbb E_{\mathbf{x}_0,t,\boldsymbol\epsilon}
\left[
\left\|
\boldsymbol\epsilon-
\boldsymbol\epsilon_\theta(\mathbf{x}_t,t,\mathbf c)
\right\|_2^2
\right].
\end{equation}

At inference time, generation is initialized from $\mathbf{x}_T\sim\mathcal N(\mathbf0,\mathbf I)$ and proceeds through successive reverse denoising transitions to $\mathbf{x}_0$. Deterministic non-Markovian samplers provide an alternative sampling trajectory \citep{song2021ddim}, while latent diffusion performs the same process in the latent space of a pretrained autoencoder \citep{rombach2022high}.

\subsection{Data Attribution}
\label{sec:data-attribution}

Given a training set $\mathcal D=\{x^i\}_{i=1}^{N}$ and model parameters $\theta_{\mathcal D}$ trained on it, data attribution quantifies the contribution of each training example to a model behavior of interest. Let $\mathcal V(x_{\mathrm{test}};\theta)$ denote a scalar utility. A direct counterfactual definition is the leave-one-out effect: 
\begin{equation}
\Delta_i(x_{\mathrm{test}})
=
\mathcal V(x_{\mathrm{test}};\theta_{\mathcal D\setminus i})
-
\mathcal V(x_{\mathrm{test}};\theta_{\mathcal D}),
\label{eq:model-loo}
\end{equation}
where $\theta_{\mathcal D\setminus i}$ denotes the model retrained after removing $x^i$. Repeating this procedure for every training example is computationally prohibitive.

Influence functions approximate this effect through a first-order parameter perturbation \citep{koh2017influence}. Let
$g_i=\nabla_\theta\mathcal L(x^i;\theta_{\mathcal D})$,
$g_{\mathcal V}=\nabla_\theta\mathcal V(x_{\mathrm{test}};\theta_{\mathcal D})$, and
$H=\nabla_\theta^2\mathcal L(\mathcal D;\theta_{\mathcal D})$.
Up to sign and dataset-size conventions, the influence score is: 
\begin{equation}
\mathcal I(x^i,x_{\mathrm{test}})
=
g_{\mathcal V}^{\top}H^{-1}g_i.
\label{eq:if}
\end{equation}
Here, $H^{-1}g_i$ approximates the parameter change induced by reweighting the training example, while $g_{\mathcal V}$ measures the effect of that change on the target utility.

% \subsection{Data Attribution for Diffusion Models}
% \label{sec:diffusion-attribution}

For diffusion models, the target behavior is produced through a sequence of denoising states, making the generation process inherently trajectory-dependent. Consequently, training-data relevance may vary across diffusion time and internal network representations. Diffusion data attribution therefore seeks to identify the training examples associated with a generated output while accounting for the structured, multi-step nature of the generation process.

\section{Method}
\label{sec:method}

\begin{figure*}[t]
    \centering
    \includegraphics[width=.8\textwidth]{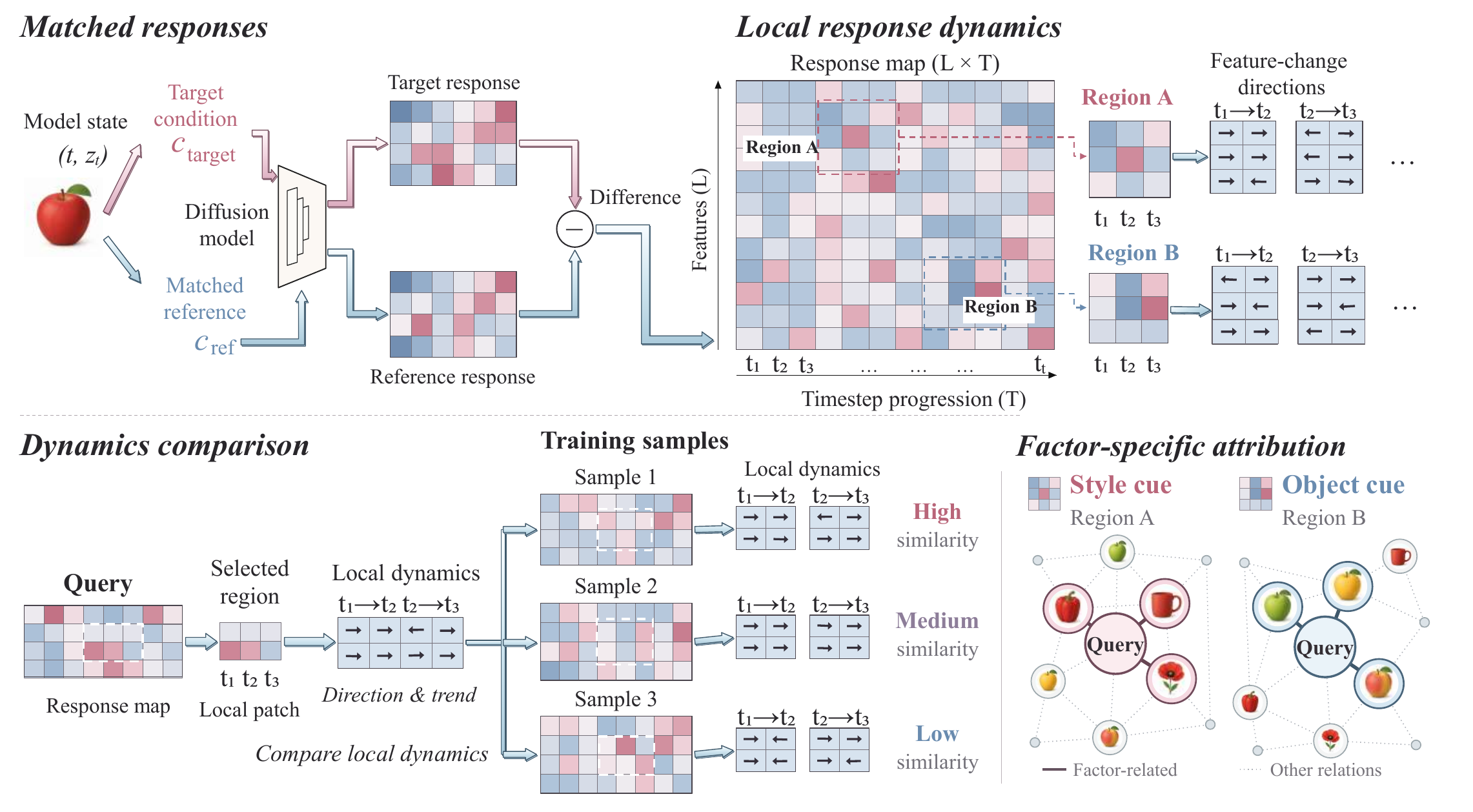}
    \vspace{-10pt}
    \caption{
        The overview of our proposed \method{}. Throughout the experiments, the term "style" denotes either color or image style.
        % : matched counterfactual responses are converted
        % into dynamic trajectory representations and compared under a
        % training-bank-adaptive kernel for factor-specific attribution.
    }
    \vspace{-15pt}
    \label{fig:rda-overview}
\end{figure*}

\subsection{Overview}
\label{sec:dynamic-attribution}
We introduce \emph{\methodfull{}} (\method), a trajectory-aware framework for
factor-specific data attribution in diffusion models. Let $f_\theta$ denote a
diffusion model trained on $\mathcal D=\{(x_i,y_i)\}_{i=1}^{N}$ and fixed
during attribution. Given a generated query $q$ and a target factor $c$, our
goal is to rank training examples according to their relevance to the
occurrence of $c$ in $q$. Specifically, \method{} proceeds in three stages. First, it constructs factor-specific counterfactual responses by comparing target and reference conditions at
matched noisy states. Second, stage-wise dynamic features are extracted from
these responses and concatenated over selected denoising timesteps into
trajectory representations. Finally, local descriptor blocks are calibrated
using the training bank and compared through a covariance-aware
positive-semidefinite kernel. 
% Detailed definitions of the dynamic features are provided in Appendix~\ref{app:dynamic-descriptor}.

\subsection{Factor-Specific Counterfactual Response}
\label{sec:counterfactual-response}

Let $y_c$ denote the conditioning input associated with the target factor $c$. We define $y_{\mathrm{ref}(c)}$ as its matched counterfactual reference, obtained by replacing the target factor $c$ with a predefined reference value while keeping all other conditioning information unchanged.
Let \(\{t_j\}_{j=0}^{T-1}\) denote the recorded diffusion timesteps, where \(t_j\) is the \(j\)-th step and \(T\) is the total number of recorded steps,
$s\in\{1,\ldots,S\}$ index repeated analysis probes, and
$\ell\in\mathcal{I}_{\mathrm{layer}}$ index the selected network stages. For sample $x$,
let $x_{t_j}^{(s)}$ denote the corresponding noisy state under probe $s$. Given a noisy state $x_{t_j}^{(s)}$, timestep $t_j$, and conditioning input $y_c$, we perform a standard denoiser forward pass and expose the intermediate
representations at the selected network stages as:
\begin{equation}
\left(
\hat{\epsilon}_\theta(x_{t_j}^{(s)},t_j,y_c),
\left\{H_\ell(x_{t_j}^{(s)},t_j,y_c)\right\}_{\ell\in\mathcal{I}_{\mathrm{layer}}}
\right)
=
f_\theta^{\mathrm{feat}}(x_{t_j}^{(s)},t_j,y_c),
\label{eq:feature-forward}
\end{equation}
where $f_\theta^{\mathrm{feat}}$ denotes the original denoiser with access to selected intermediate features. Each activation is extracted as
$H_\ell(x_{t_j}^{(s)},t_j,y_c)\in\mathbb R^{P_\ell\times C_\ell}$,
where $P_\ell$ and $C_\ell$ denote the token and channel dimensions, respectively.
% where $f_\theta^{\mathrm{feat}}$ denotes the original denoiser augmented only
% to expose its selected intermediate representations. Each recorded activation
% is converted by the same fixed representation extraction procedure into
% % \begin{equation}
% $H_\ell(x_{t_j}^{(s)},t_j,y_c)\in\mathbb R^{P_\ell\times C_\ell}$,
% % \end{equation}
% where $P_\ell$ is the token count and $C_\ell$ denote channel dimensions at layer $\ell$.

At each recorded state, we evaluate this representation under the target
condition $y_c$ and the matched reference condition
$y_{\mathrm{ref}(c)}$ while keeping the noisy state fixed. We define the
factor-specific counterfactual response as:
\begin{equation}
\mathcal R_{c,s}^{\ell}(x,t_j)
=
H_\ell(x_{t_j}^{(s)},t_j,y_c)
-
H_\ell(x_{t_j}^{(s)},t_j,y_{\mathrm{ref}(c)}).
\label{eq:counterfactual-response}
\end{equation}

Both evaluations share the same protocol, only the conditioning signal changes. Their difference therefore provides a local counterfactual response associated with the target
factor while controlling for the underlying diffusion state. Repeating this
matched intervention over the recorded timesteps, selected layers, and
analysis probes yields the factor-specific response sequence
$\{\mathcal R_{c,s}^{\ell}(x,t_j)\}_{s,\ell,j}$, from which we construct the
dynamic trajectory representation.

\subsection{Dynamic Trajectory Construction}
\label{sec:dynamic-signature}

Given the counterfactual responses obtained above, we construct two
complementary dynamic components directly from the response sequence. For each
layer $\ell$ and timestep $t_j$, the response-strength component is computed
from the counterfactual responses at the endpoint state as:
\begin{equation}
\psi^E_{\ell,t_j}(x)
=
\mathcal E\!\left(
\left\{
\mathcal R_{c,s}^{\ell}(x,t_{j})
\right\}_{s=1}^{S}
\right),
\label{eq:energy-from-response}
\end{equation}
while the directional component is computed from the evolution of the
counterfactual responses across the interval as:
\begin{equation}
\psi^V_{\ell,t_j,k}(x)
=
\mathcal V_k\!\left(
\left\{
\mathcal R_{c,s}^{\ell}(x,t_j),
\mathcal R_{c,s}^{\ell}(x,t_{j+1})
\right\}_{s=1}^{S}
\right),
\label{eq:velocity-from-response}
\end{equation}
where $\mathcal E$ extracts response strength and $\mathcal V_k$ extracts the
first-order directional evolution of channel $k$. 

Let $\boldsymbol{\tau}
=
(\tau_1,\tau_2,\ldots,\tau_M)$ denote the ordered sequence of selected timesteps along the denoising
trajectory, where \(M\) is the number of selected stages. For layer \(\ell\)
and selected timestep \(\tau_m\), we construct the stage-wise dynamic feature
as:
\begin{equation}
\psi_{\ell,\tau_m}(x)
=
\operatorname{Concat}
\left(
\psi^{E}_{\ell,\tau_m}(x),
\psi^{V}_{\ell,\tau_m,1}(x),
\ldots,
\psi^{V}_{\ell,\tau_m,C_\ell}(x)
\right),
\label{eq:stage-dynamic-feature}
\end{equation}
where
\(\psi^E_{\ell,\tau_m}(x)\in\mathbb R^{C_\ell}\)
denotes the response-strength feature at timestep \(\tau_m\), and
\(\psi^V_{\ell,\tau_m,k}(x)\in\mathbb R^{P_\ell}\)
denotes the directional feature for channel \(k\) over the corresponding
denoising transition. The complete definitions of these components, including
their temporal differences, multi-probe stability, and second-order
reliability, are provided in
\textbf{Appendix}~\ref{app:dynamic-descriptor}.

The stage-wise features are then concatenated following the denoising order to
form the layer-specific trajectory descriptor as:
\begin{equation}
z_\ell(x)
=
\operatorname{Concat}
\left(
\psi_{\ell,\tau_1}(x),
\psi_{\ell,\tau_2}(x),
\ldots,
\psi_{\ell,\tau_M}(x)
\right).
\label{eq:layer-trajectory}
\end{equation}

Similarly, let $\boldsymbol{\ell}
=
(\ell_1,\ell_2,\ldots,\ell_L)$ denote the ordered sequence of selected layers. The corresponding
multi-layer trajectory representation is defined as:
\begin{equation}
\Phi_c(x)
=
\operatorname{Concat}
\left(
z_{\ell_1}(x),
z_{\ell_2}(x),
\ldots,
z_{\ell_L}(x)
\right).
\label{eq:trajectory-representation}
\end{equation}

This construction preserves the ordering of dynamic information
across denoising stages rather than averaging it over time, and is applied
identically to training examples and generated queries.

\subsection{Training-Bank Adaptive Comparison}
\label{sec:bank-geometry}

Direct comparison of the complete trajectory representations can be dominated
by variations broadly expressed across the training set. We therefore
calibrate local feature blocks using statistics estimated from the training
bank. Let $\mathcal B=\mathcal B_E\cup\mathcal B_V$ denote the set of kernel blocks,
where $\mathcal B_E$ contains response-strength blocks and $\mathcal B_V$
contains directional blocks. Specifically, an energy block contains the
channel-wise feature $\psi^E_{\ell,t_j}(x)$ for a fixed layer and state, whereas
a directional block contains $\psi^V_{\ell,t_j,k}(x)$ for a fixed layer,
timestep, and channel.

For $b\in\mathcal B$, let $\Pi_b$ denote the coordinate-selection operator that
extracts the corresponding raw block from $\Phi_c(x)$, which can be formulated as
% \begin{equation}
$\xi_b(x)
=
\Pi_b\Phi_c(x)$.
\label{eq:raw-kernel-block}
% \end{equation}
We use $\Pi_b$ rather than $P_b$ to distinguish this operator from the spatial
dimension $P_\ell$.

Each block is centered and standardized using statistics computed exclusively
from the training examples. Let $\mathcal S_b$ denote this frozen
standardization operator, including the removal of coordinates with negligible
training variance, and define
% \begin{equation}
$h_b(x)
=
\mathcal S_b\!\left(\xi_b(x)\right)
\in\mathbb R^{d_b}$,
\label{eq:kernel-block}
% \end{equation}
where $d_b$ is the number of retained coordinates in block $b$.

The standardized training bank for this block can be formulated as:
\begin{equation}
Z_b
=
\begin{bmatrix}
h_b(x_1)^\top\\
\vdots\\
h_b(x_N)^\top
\end{bmatrix}
\in\mathbb R^{N\times d_b}.
\label{eq:block-bank}
\end{equation}
Because $h_b$ is centered using training-bank statistics,
$\frac{1}{N} Z_b^\top Z_b$ corresponds to the empirical covariance of the
standardized block features. We define the regularized inverse: 
\begin{equation}
G_b
=
\left(
Z_b^\top Z_b+\lambda I_{d_b}
\right)^{-1},
\qquad
\lambda>0,
\label{eq:bank-geometry}
\end{equation}
where $I_{d_b}$ is the $d_b$-dimensional identity matrix and $\lambda$
controls regularization.

The final attribution score between query $q$ and training example $x_i$ is: 
\begin{equation}
\boxed{
K_\lambda(q,x_i)
=
\sum_{b\in\mathcal B}
w_b\,
h_b(q)^\top
G_b
h_b(x_i)
},
\qquad
w_b\ge0.
\label{eq:rda-kernel}
\end{equation}
Here, $w_b$ denotes a fixed block weight used to balance contributions from
different dynamic components and blocks. In the full method, the response-
strength and directional feature families receive equal total weight; the
exact weighting scheme is given in \textbf{Appendix}~\ref{app:dynamic-descriptor}.

$G_b$ calibrates each local component against the variation observed in the corresponding training block, and the resulting blockwise interactions are aggregated to rank the training examples. All training-bank statistics, standardization operators, and $G_b$ are fitted once on the training set and remain fixed during query attribution.

\subsection{Representation-Level Removal Interpretation}
\label{sec:removal-connection}

Because each training example contributes a descriptor to every block bank,
the attribution score also admits a deletion-based interpretation. While
deletion and retraining provide a model-level counterfactual
(\textbf{Appendix}~\ref{app:artbench-deletion}), here we ask how removing
$x_i$ affects query reconstruction in the fixed representation space.

For block $b$, consider the ridge reconstruction objective: 
\begin{equation}
E_b(q)
=
\min_{a_b\in\mathbb R^N}
\|h_b(q)-Z_b^\top a_b\|_2^2
+
\lambda\|a_b\|_2^2,
\label{eq:block-reconstruction}
\end{equation}
where $a_b$ contains the reconstruction coefficients. If $a_b^\star$ is the
minimizer, its $i$-th entry is: 
\begin{equation}
a_{ib}^{\star}
=
h_b(q)^\top G_b h_b(x_i),
\label{eq:block-ridge-coefficient}
\end{equation}
which is exactly the blockwise interaction underlying $K_\lambda$.

Removing $x_i$ from each block bank changes the optimal reconstruction
objective by: 
\begin{equation}
\Delta_i^{\mathrm{dyn}}(q)
=
\lambda
\sum_{b\in\mathcal B}
w_b
\frac{
\left[h_b(q)^\top G_b h_b(x_i)\right]^2
}{
1-h_b(x_i)^\top G_b h_b(x_i)
}.
\label{eq:dyn-removal-brief}
\end{equation}
Thus, the interactions used by $K_\lambda$ also determine an exact
leave-one-descriptor-out effect, distinct from deleting a training example and
retraining the diffusion model. The full derivation is provided in
\textbf{Appendix}~\ref{app:theory}.

\section{Experiments}
\label{sec:experiments}

\subsection{Evaluation protocol and scope}

\begin{wrapfigure}[10]{r}{0.5\linewidth}
    \centering
    \vspace{-8mm}
    \includegraphics[
        width=\linewidth,
        trim={0cm 5cm 0cm 1.5cm},
        clip
    ]{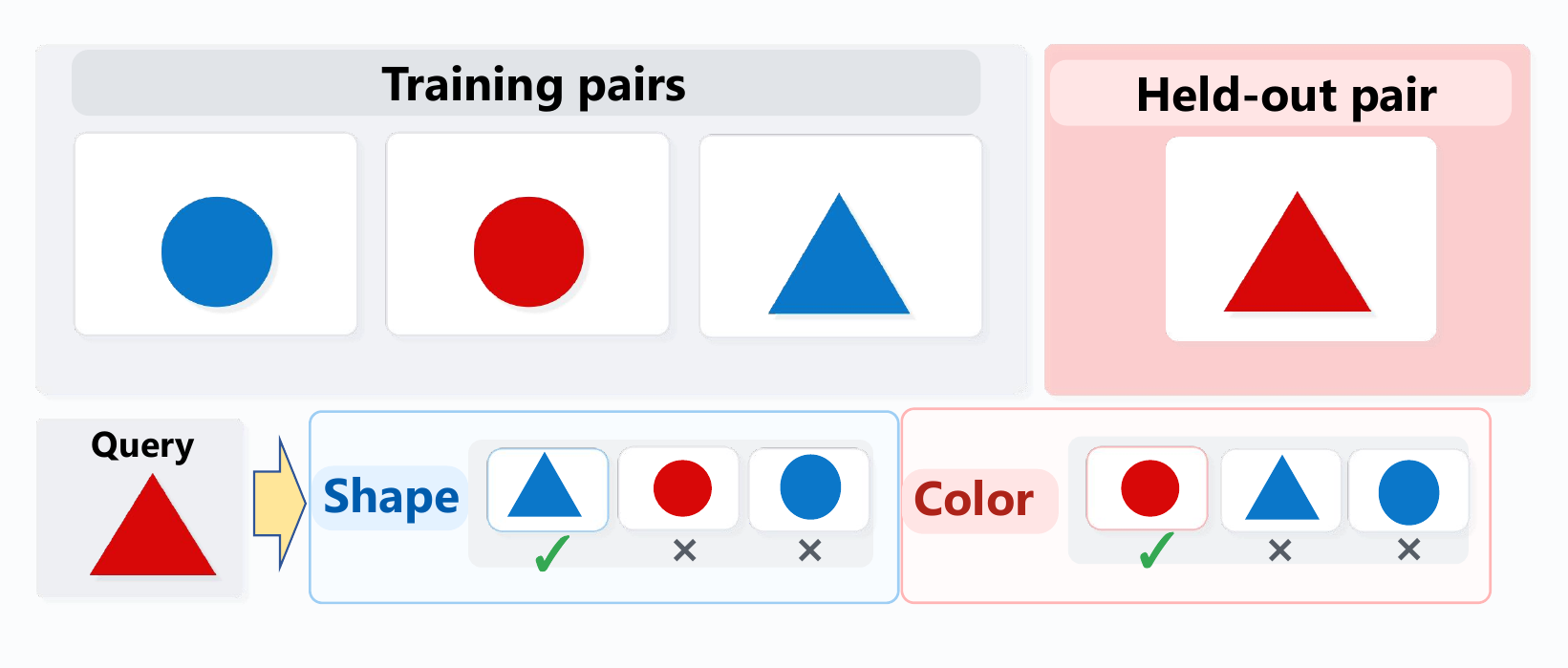}
    \vspace{-15pt}
    \caption{Synthetic setup for shape and color attribution.}
    \label{fig:synthetic-setup}
    % \vspace{-5mm}
\end{wrapfigure}

We evaluate \method{} on controlled synthetic data, CIFAR-100, and ArtBench, and compare against Concept-TRAK \citep{park2025concept}, TRAK \citep{park2023trak}, D-TRAK \citep{zheng2024dtrak}, DAS \citep{lin2025das}, and Journey-TRAK \citep{georgiev2023journey}. Within each experiment, all methods use the same checkpoint, candidate pool, queries, relevance targets, and validity mask; baseline settings follow disjoint validation data or published protocols. We report P@$K$, the fraction of top-$K$ candidates matching the target factor. The implementation details are provided in
\textbf{Appendix}~\ref{app:protocol}.

\subsection{Controlled Experiments on Synthetic Data}
\label{sec:synthetic}

% We train a conditional diffusion U-Net on blue triangles, blue circles, and
% red circles, holding out red triangles. Shape and color conditions are
% factorized, allowing a red-triangle query to be attributed separately by
% contrasting triangle vs.\ circle or red vs.\ blue while preserving the other
% factor. Figure~\ref{fig:synthetic-setup} illustrates this controlled setup and
% the corresponding factor-specific attribution targets. TRAK, D-TRAK, and DAS return a single whole-image ranking, whereas
% \method{} constructs factor-specific response trajectories using the
% target/reference procedure in Section~\ref{sec:counterfactual-response}.
% Extraction, query splitting, and support-selection details are provided in
% \textbf{Appendix}~\ref{app:protocol}.
This synthetic experiment tests whether attribution can separate the influence of different semantic factors within a compositional query. We train a conditional diffusion U-Net on blue triangles, blue circles, and red circles, while holding out red triangles. Because shape and color are factorized, each factor can be changed while keeping the other fixed. In each matched target--reference pair, effects associated with the preserved factor are shared and largely cancel, leaving a response dominated by the changed factor. This yields distinct shape- and color-specific attribution signals, and Figure~\ref{fig:synthetic-setup} illustrates the setup and corresponding attribution targets. The details are provided in \textbf{Appendix}~\ref{app:protocol}.

\begin{wraptable}[9]{l}{0.5\columnwidth}
% \vspace{-9pt}
\centering
\caption{Synthetic held-out OOD precision (\%).}
\vspace{-5pt}
\label{tab:synthetic}
\scriptsize
\setlength{\tabcolsep}{1.8pt}
\renewcommand{\arraystretch}{0.92}
\resizebox{\linewidth}{!}{
\begin{tabular}{lrrrrrrrrr}
\toprule
& \multicolumn{3}{c}{\textbf{Shape}}
& \multicolumn{3}{c}{\textbf{Color}}
& \multicolumn{3}{c}{\textbf{Avg.}} \\
\cmidrule(lr){2-4}
\cmidrule(lr){5-7}
\cmidrule(lr){8-10}
Method
& P@1 & P@5 & P@10
& P@1 & P@5 & P@10
& P@1 & P@5 & P@10 \\
\midrule
TRAK
& 71.20 & 41.10 & 33.64
& 26.20 & 47.54 & 50.96
& 48.70 & 44.32 & 42.30 \\
D-TRAK
& 0.30 & 0.68 & 1.00
& \best{99.70} & \best{99.32} & \best{99.00}
& 50.00 & 50.00 & 50.00 \\
DAS
& 65.30 & 43.44 & 37.27
& 34.70 & 53.70 & 50.52
& 50.00 & 48.57 & 43.90 \\ \hline
\method{}
& \best{91.20} & \best{91.36} & \best{92.01}
& 90.80 & 93.88 & 93.65
& \best{91.00} & \best{92.62} & \best{92.83} \\
\bottomrule
\end{tabular}
}

\vspace{-5pt}
\end{wraptable}

% Table~\ref{tab:synthetic} shows that \method{} retrieves shape and color
% consistently, whereas whole-image baselines exhibit strong factor preference:
% TRAK primarily favors shape, while D-TRAK favors color. This contrast
% demonstrates the ambiguity of using a single ranking when multiple factors
% coexist in the same generation.
Table~\ref{tab:synthetic} shows that \method{} achieves consistently high precision for both shape and color, whereas whole-image baselines exhibit strong factor imbalance. D-TRAK nearly perfectly retrieves color but fails on shape, while TRAK and DAS are substantially more shape-biased. This indicates that a single attribution ranking can be dominated by one factor, whereas factor-specific contrasts provide more balanced and disentangled attribution.

\begin{wrapfigure}[17]{r}{0.48\linewidth}
    \centering
    \vspace{-4mm}

    \includegraphics[
        width=0.90\linewidth
    ]{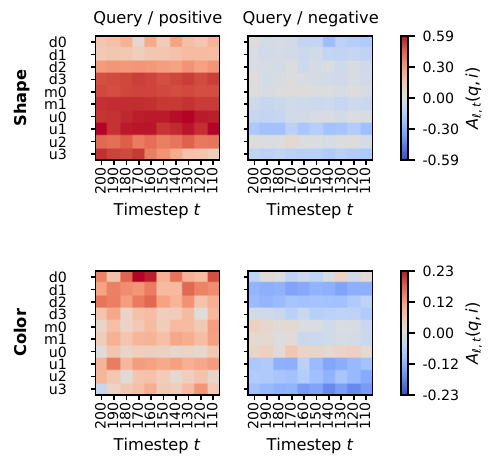}

    \vspace{-5pt}
    \caption{
    Signed alignment of the query with positive and negative training examples.
    }
    \label{fig:synthetic-local-alignment}

    \vspace{-4mm}
\end{wrapfigure}

% A key challenge in factor-specific attribution is identifying training examples that consistently support the target factor throughout denoising. 
% Figure~\ref{fig:synthetic-local-alignment} shows that positive examples exhibit stronger directional agreement with the query, while negative examples show weaker or opposing dynamics, explaining why \method{} benefits from modeling trajectory direction.
Figure~\ref{fig:synthetic-local-alignment} shows that positive training examples maintain stronger and more consistent directional agreement with the query across layers and timesteps, whereas negative examples exhibit weaker or frequently opposing responses. This distinction indicates that attribution evidence is carried not only by the presence of a response, but also by whether its direction remains compatible with the query throughout denoising. By explicitly modeling this directional evolution, \method{} can preserve the temporal consistency of factor-specific responses and better separate supportiv training examples from conflicting ones.

\begin{figure*}[!b]
    \centering

    \begin{minipage}[t]{0.48\textwidth}
        \centering
        \includegraphics[width=0.92\linewidth]
        {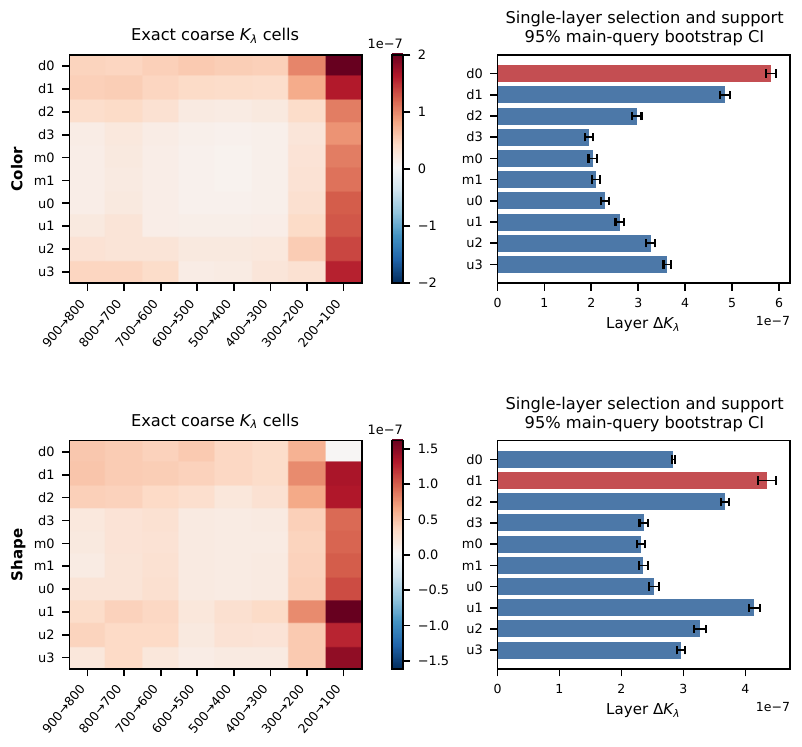}

        \vspace{-5mm}
        \begin{minipage}[t]{0.46\linewidth}
            \centering
            \footnotesize \textbf{(a)} Layer--timestep support
        \end{minipage}
        \hfill
        \begin{minipage}[t]{0.46\linewidth}
            \centering
            \footnotesize \textbf{(b)} Single-layer support
        \end{minipage}

        \vspace{-1mm}
        \captionof{figure}{
        Layer-wise attribution support of $K_\lambda$ for color and shape.
        \textbf{(a)} Coarse layer--timestep contributions.
        \textbf{(b)} Aggregated single-layer support with bootstrap confidence timestep.
        }
        \label{fig:synthetic-layer-support}
    \end{minipage}
    \hfill
    \begin{minipage}[t]{0.48\textwidth}
        \centering
        \includegraphics[width=0.92\linewidth]
        {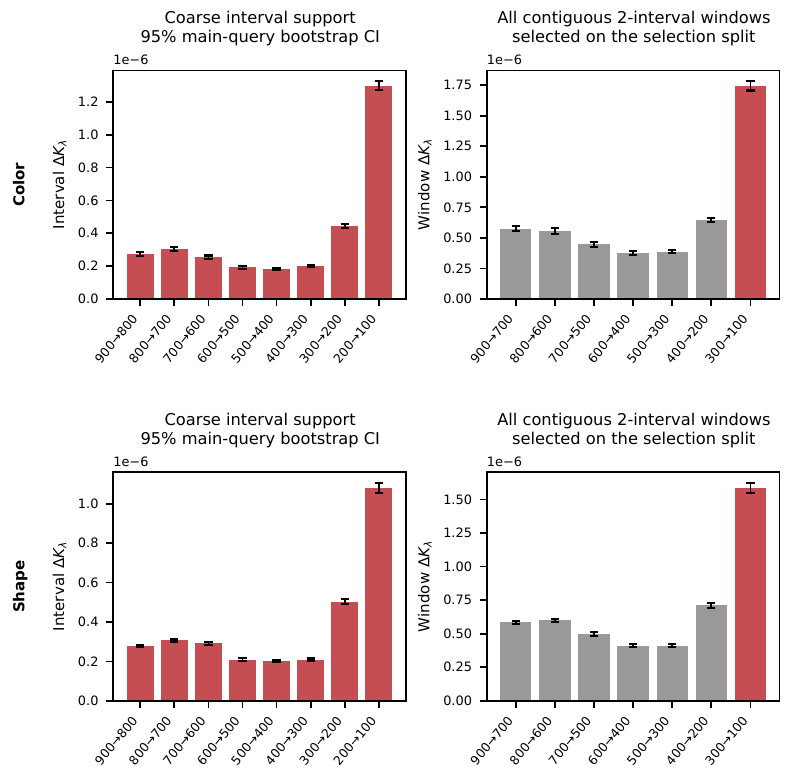}

        \vspace{-1.5mm}
        \begin{minipage}[t]{0.46\linewidth}
            \centering
            \footnotesize \textbf{(a)} Interval support
        \end{minipage}
        \hfill
        \begin{minipage}[t]{0.46\linewidth}
            \centering
            \footnotesize \textbf{(b)} Window support
        \end{minipage}

        \vspace{-1mm}
        \captionof{figure}{
        Temporal attribution support of $K_\lambda$ for color and shape.
        \textbf{(a)} Support over individual coarse diffusion timestep.
        \textbf{(b)} Support over contiguous two-timestep windows.
        }
        \label{fig:synthetic-temporal-support}
    \end{minipage}

    \vspace{-4mm}
\end{figure*}

We further find that factor-specific evidence is distributed across both network depth and denoising time, rather than localized to a single representation. Figure~\ref{fig:synthetic-layer-support} reveals distinct layer-wise support patterns for color and shape, while Figure~\ref{fig:synthetic-temporal-support} identifies the denoising interval from (t=300) to (t=100) as the strongest temporal region for both.
% Because the same factor-specific target--reference contrast can be evaluated at each layer and timestep, \method{} can trace how the corresponding evidence emerges and evolves throughout denoising, and aggregate the informative responses across these representations.
After obtaining a factor-specific response at each denoising stage, \method{} organizes these stage-wise responses into an ordered trajectory within each layer. It then models how their magnitude and direction evolve between successive timesteps, converting isolated local responses into a dynamic signature that spans the denoising process.

\subsection{Hierarchical Concept Attribution}
\label{sec:cifar}

\begin{wraptable}[8]{r}{0.4\columnwidth}
\vspace{-15pt}
\centering
\caption{CIFAR-100 P@5 (\%).}
\vspace{-5pt}
\label{tab:cifar}
\footnotesize
% \setlength{\tabcolsep}{3.8pt}
% \renewcommand{\arraystretch}{0.98}
% \begin{adjustbox}{max width=\linewidth}
\begin{tabular}{lrr}
\toprule
Method & Mid-level & Fine class \\
\midrule
D-TRAK & 45.30 & 18.54 \\
TRAK & 41.35 & 15.30 \\
DAS & 11.95 & 1.62 \\
Journey-TRAK & 10.76 & 1.95 \\ \hline
\method{} & \best{50.59} & \best{21.68} \\
\bottomrule
\end{tabular}
% \end{adjustbox}
% \vspace{-10pt}
\end{wraptable}
Beyond the controlled synthetic setting, we further evaluate \method{} on the more complex and realistic CIFAR-100 benchmark \citep{krizhevsky2009cifar}. We train the diffusion model using only coarse-grained labels and evaluate attribution retrieval at both the mid-level superclass and fine-class levels, testing whether meaningful hierarchical attribution can emerge beyond the supervision granularity seen during training. We use a fixed candidate bank with disjoint query and validation banks. Further architecture and extraction details are provided in \textbf{Appendix}~\ref{app:protocol}.

% Table~\ref{tab:cifar} compares \method{} with published baselines on the same queries. \method{} performs best at both hierarchy levels, with a clearer separation at the superclass level. The improvement over gradient-based baselines supports the combination of factor-conditioned dynamics.

Table~\ref{tab:cifar} compares \method{} with published baselines on the same queries. \method{} achieves the strongest retrieval at both hierarchy levels, with a larger margin at the superclass level, where semantically related subclasses share more coherent response dynamics. Fine-class attribution is more challenging because it requires distinguishing samples within the same coarse training category, yet \method{} retains a clear advantage despite the absence of fine-grained supervision during diffusion training. This suggests that the factor-conditioned trajectories capture hierarchical structure beyond the training labels, while the bank-adaptive comparison helps preserve the more discriminative dynamic variations needed for finer attribution.

% Prevent the following wide table from floating above this subsection.
\FloatBarrier

% =========================================================
% CIFAR-100 label-source recovery
% =========================================================

\subsection{Counterfactual label-source recovery}
\label{sec:cifar-label-flip}

\begin{wrapfigure}{r}{0.5\linewidth}
  \centering
  \vspace{-4mm}
  \includegraphics[width=\linewidth]{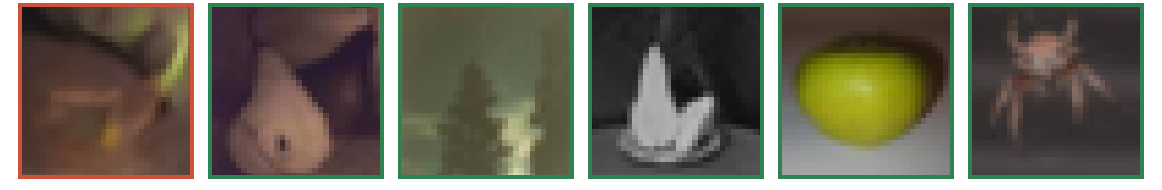}\\[-0.5mm]
  \includegraphics[width=\linewidth]{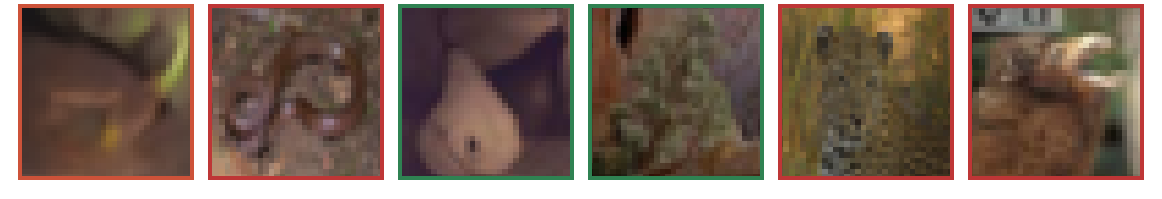}\\[-0.5mm]
  \vspace{-2mm}
  \caption{Label-corruption retrieval for the same query, showing top-five training examples. Red: corrupted animals; green: genuine plants. Top: D-TRAK; bottom: \method{}.}
  \vspace{-15pt}
  \label{fig:cifar-label-flip}
\end{wrapfigure}

Having evaluated hierarchical attribution under clean supervision, we next stress-test whether the retrieved sources remain meaningful when training labels conflict with visual content. We relabel a subset of animal images as plant and ask whether a plant-conditioned generation can be attributed not only to genuine plant examples but also to the deliberately mislabeled animal sources. This setting distinguishes label-driven attribution from attribution that reflects the actual training influence encoded in the model. Relevance is defined by the original group before corruption.

% As shown in Table~\ref{tab:cifar-label-flip} and Fig.~\ref{fig:cifar-label-flip}, \method{} consistently retrieves genuine plant sources while also recovering corrupted animal sources. These results indicate that the modeled response trajectories retain provenance information beyond the observed labels, although corrupted sources remain more difficult to recover.
Table~\ref{tab:cifar-label-flip} shows that \method{} maintains strong recovery for both source types under corrupted supervision. It achieves high precision on genuine plants while recovering substantially more flipped animal sources than most baselines, leading to the strongest overall balance across the two groups. This indicates that the attribution signal is not determined solely by the observed plant label: the internal denoising response trajectories clearly reflect the underlying source identity even when the observed label conflicts with the original source. \method{} directly exploits this property by modeling these response trajectories for attribution.

Figure~\ref{fig:cifar-label-flip} provides a qualitative view of the same behavior. For a single plant-conditioned query, \method{} retrieves both genuine plant examples and animals deliberately relabeled as plant. The retrieved mixture illustrates how the dynamic attribution signal can trace both the observed conditioning association and the corrupted label.

\begin{center}
\begin{minipage}{\linewidth}
\centering
\vspace{-5PT}
\captionof{table}{CIFAR-100 20\% label-flip source recovery (\%). Values within each cell are P@$1/5/10/50$.}
\label{tab:cifar-label-flip}

\vspace{-5pt}
\small
\setlength{\tabcolsep}{3.8pt}

\resizebox{\linewidth}{!}{
\begin{tabular}{lccc}
\toprule
Method
& Avg.
& Flipped animal $\rightarrow$ plant
& Genuine plant \\
\midrule

TRAK
& 56.22 / 52.00 / 50.14 / 44.27
& 23.24 / 21.95 / 20.00 / 16.40
& 89.19 / 82.05 / 80.27 / 72.14 \\

D-TRAK
& 63.78 / 59.89 / 57.68 / 49.12
& \best{40.54} / 35.68 / 33.51 / 26.30
& 87.03 / 84.11 / 81.84 / 71.94 \\

DAS
& 24.05 / 21.46 / 19.89 / 22.28
& 1.08 / 0.76 / 1.03 / 4.17
& 47.03 / 42.16 / 38.76 / 40.39 \\

Journey-TRAK
& 1.35 / 6.59 / 22.57 / 21.75
& 1.08 / 1.51 / 13.51 / 9.01
& 1.62 / 11.68 / 31.62 / 34.49 \\

\midrule

\method{}
& \best{67.57} / \best{66.86} / \best{66.51} / \best{64.48}
& 37.84 / \best{38.27} / \best{37.51} / \best{34.42}
& \best{97.30} / \best{95.46} / \best{95.51} / \best{94.53} \\

\bottomrule
\end{tabular}
}
\vspace{-5pt}
\end{minipage}
\end{center}

\subsection{Style and object--style attribution}
\label{sec:artbench}

% \Needspace{0.25\textheight}
\begin{wraptable}[12]{r}{0.45\columnwidth}
\vspace{-15pt}
\centering
\caption{Generated-query ArtBench style precision (\%).}
\vspace{-5pt}
\label{tab:artbench}
\scriptsize
\setlength{\tabcolsep}{2.3pt}
\renewcommand{\arraystretch}{0.90}
\resizebox{\linewidth}{!}{
\begin{tabular}{llrrrr}
\toprule
Dataset & Method & P@1 & P@5 & P@10 & P@50 \\
\midrule
\multirow{5}{*}{ArtBench2}
& TRAK & 98.54 & 97.49 & 96.58 & 89.15 \\
& D-TRAK & 97.80 & 97.72 & 96.99 & 94.50 \\
& Journey-TRAK & 98.01 & 96.76 & 95.91 & 89.07 \\
& DAS & 99.37 & 99.46 & 99.32 & 99.03 \\ 
\cmidrule(lr){2-6}
& \method{} & \best{100.00} & \best{100.00} & \best{100.00} & \best{100.00} \\
\midrule
\multirow{5}{*}{ArtBench5}
& TRAK & 90.01 & 87.18 & 84.36 & 70.36 \\
& D-TRAK & 63.00 & 62.96 & 62.57 & 58.32 \\
& Journey-TRAK & 89.19 & 83.59 & 78.92 & 62.32 \\
& DAS & 96.43 & 95.49 & 94.52 & 91.28 \\ 
\cmidrule(lr){2-6}
& \method{} & \best{100.00} & \best{100.00} & \best{100.00} & \best{99.94} \\
\bottomrule
\end{tabular}
}
\vspace{-25pt}
\end{wraptable}

We further evaluate \method{} on ArtBench \citep{liao2022artbench}, extending the study to a more complex text-conditioned diffusion setting with two- and five-style subsets. We consider both pure-style and object--style attribution, testing whether the method remains effective when multiple semantic factors coexist in the same generation. 
% The results can be found in Tab. xxx. It can be seen that . The qualitative results are illustrated in Fig. xxx.
The quantitative results are reported in Tables~\ref{tab:artbench} and \ref{tab:artbench-composition}. It can be seen that \method{} nearly saturates generated-query style retrieval on both subsets and also achieves strong performance in the more challenging object--style setting, including the more diverse five-style split. This advantage is consistent with the target-specific response trajectories of \method{}, which separate style-directed and object-directed evidence instead of forcing multiple factors into a single whole-image ranking. %Figure~\ref{fig:artbench-retrieval} provides representative qualitative retrieval results, showing that the retrieved training examples consistently match the target style and complementing the quantitative comparisons.

% While whole-image baselines use a single ranking across targets, \method{} constructs separate style- and object-directed attribution signals from their corresponding response trajectories. Dataset metadata defines style relevance, and all methods share the same validity mask and training order.

\begin{table}[H]
\centering
\caption{ArtBench composition precision (\%) for style- and object-specific attribution.}
\label{tab:artbench-composition}
\scriptsize
\setlength{\tabcolsep}{2.5pt}
\renewcommand{\arraystretch}{0.92}

\begin{adjustbox}{max width=\columnwidth}
\begin{tabular}{lrrrrrrrrrrrr}
\toprule
& \multicolumn{6}{c}{\textbf{ArtBench2}}
& \multicolumn{6}{c}{\textbf{ArtBench5}} \\
\cmidrule(lr){2-7}\cmidrule(lr){8-13}

& \multicolumn{3}{c}{\textbf{Style}}
& \multicolumn{3}{c}{\textbf{Object}}
& \multicolumn{3}{c}{\textbf{Style}}
& \multicolumn{3}{c}{\textbf{Object}} \\
\cmidrule(lr){2-4}\cmidrule(lr){5-7}
\cmidrule(lr){8-10}\cmidrule(lr){11-13}

Method
& P@1 & P@5 & P@10 & P@1 & P@5 & P@10
& P@1 & P@5 & P@10 & P@1 & P@5 & P@10 \\
\midrule

TRAK
& 90.91 & 90.91 & 88.18
& 0.00 & 3.64 & 3.64
& 23.33 & 20.00 & 19.67
& 0.00 & 0.00 & 1.67 \\

D-TRAK
& \best{100.00} & 90.91 & 84.55
& 0.00 & 0.00 & 1.82
& 10.00 & 25.33 & 21.33
& 0.00 & 0.00 & 0.00 \\

Journey-TRAK
& 90.91 & 92.73 & 94.55
& 0.00 & 0.00 & 0.91
& 23.33 & 16.67 & 16.67
& 0.00 & 2.00 & 3.00 \\

DAS
& 72.73 & 83.64 & 85.45
& 0.00 & 0.00 & 3.64
& 10.00 & 18.67 & 20.33
& 0.00 & 0.67 & 0.67 \\ \hline

\method{}
& \best{100.00} & \best{100.00} & \best{100.00}
& \best{9.09} & \best{9.09} & \best{6.36}
& \best{90.00} & \best{90.67} & \best{88.67}
& \best{30.00} & \best{18.67} & \best{14.67} \\

\bottomrule
\end{tabular}
\end{adjustbox}
\end{table}

% \begin{wrapfigure}[14]{r}{0.55\linewidth}
%     \centering
%     \vspace{-2mm}
%     \includegraphics[width=\linewidth]{assets/retrieval_ArtBench2_cdka-k_generated-style_post_impressionism_sample_q500.jpg}\\[1mm]
%     \includegraphics[width=\linewidth]{assets/retrieval_ArtBench5_cdka-k_generated-style_baroque_sample_q800.jpg}
%     % \vspace{-2mm}
%     \vspace{-10pt}
%     \caption{Representative style retrievals produced by \method{}. Top row: ArtBench2, post-impressionism. Bottom row: ArtBench5, baroque. Blue: query; green: style-relevant training candidate.}
%     \label{fig:artbench-retrieval}
% \end{wrapfigure}

% \method{} nearly saturates generated-query style retrieval on both ArtBench subsets and remains effective in the more challenging object--style setting, including the five-style split. Its advantage comes from constructing target-specific response trajectories, which separate style- and object-directed evidence instead of forcing both factors into a single whole-image ranking. Figures~\ref{fig:artbench-retrieval} illustrate how these factor-specific rankings map to concrete training examples, complementing the quantitative results in Tables~\ref{tab:artbench} and \ref{tab:artbench-composition}.

\FloatBarrier
\par
\WFclear

\subsection{Ablation Study}
\label{sec:ablation}

\begin{wraptable}[10]{r}{0.48\columnwidth}
\vspace{-18pt}
\centering
\caption{Ablation of descriptor components and temporal support on CIFAR-100 at P@5.}
\label{tab:cifar-descriptor-ablation}
\scriptsize
\setlength{\tabcolsep}{4pt}
\renewcommand{\arraystretch}{0.95}

\resizebox{\linewidth}{!}{
\begin{tabular}{lrr}
\toprule
Variant & Mid-level & Fine class \\
\midrule
Energy only                  & 13.73 &  2.65 \\
Energy + velocity            & 31.57 &  7.68 \\
$+$ Second-order reliability & 34.22 &  8.86 \\
$+$ Multi-seed stability     & 35.24 &  9.24 \\ \hline 
Full \method{}               & \best{50.59} & \best{21.68} \\
\bottomrule
\end{tabular}
}

\vspace{-8pt}
\end{wraptable}

We ablate the dynamic descriptor and temporal support of \method{} on CIFAR-100 while keeping the evaluation protocol and $K_\lambda$ fixed. 
% Adding velocity to energy yields the main gain, with second-order reliability and multi-seed stability providing further improvements. The full \method{} performs best, supporting the value of dynamic representation.
Energy captures where the factor-specific response is strong, while velocity complements it by describing how that response changes across adjacent denoising stages. Their
combination therefore preserves both response magnitude and temporal evolution,
which explains the large gain from adding velocity. Second-order reliability
and multi-seed stability further suppress transient or noise-sensitive responses,
yielding more consistent trajectory evidence. The full \method{} performs best,
supporting the benefit of jointly modeling response dynamics and their stability.

\section{Conclusion}

% We introduced \method{}, a factor-specific diffusion data attribution framework that constructs matched counterfactual responses and represents their evolution through dynamic features along the denoising trajectory. A training-bank-adaptive kernel then compares these dynamic signatures against training examples. Experiments on multiple public datasets show that tracing how factor-specific responses evolve throughout the generation process provides effective signals for attribution to training data.
We identify a key limitation of diffusion data attribution: a single scalar score can obscure both which semantic factor is being supported and how that influence evolves during generation. We introduced \method{}, which addresses this by isolating factor-specific counterfactual responses and tracing their evolution across the denoising trajectory. Its effectiveness comes from controlling shared variation through matched target--reference comparisons while preserving informative changes in response strength and direction. More broadly, our results suggest that training influence in generative models is better viewed as a dynamic process rather than a static association. Attribution should therefore ask not only \emph{which} examples matter, but also \emph{how} their influence unfolds during generation. Future work will explore automatically discovered factors and references, broader generative architectures, and stronger causal validation of trajectory-based attribution.

\newpage

\subsection*{AI use statement}

In this work, we used generative AI tools for proofreading the manuscript and improving its readability and clarity. All research ideas, methodological and theoretical developments, experiments, and method implementations were developed and conducted by the authors. We have reviewed all AI-assisted work. We take responsibility for the final content of this work, including text, claims or artifacts produced with the aid of generative AI.

\bibliography{iclr2027_conference}

\begin{thebibliography}{27}
\providecommand{\natexlab}[1]{#1}
\providecommand{\url}[1]{\texttt{#1}}
\expandafter\ifx\csname urlstyle\endcsname\relax
  \providecommand{\doi}[1]{doi: #1}\else
  \providecommand{\doi}{doi: \begingroup \urlstyle{rm}\Url}\fi

\bibitem[Bai et~al.(2025)Bai, Cai, Chen, Chen, Chen, Cheng, Deng, Ding, Gao, Ge, Ge, Guo, Huang, Huang, Huang, Hui, Jiang, Li, Li, Li, Li, Lin, Lin, Liu, Liu, Liu, Liu, Liu, Liu, Lu, Luo, Lv, Men, Meng, Ren, Ren, Song, Sun, Tang, Tu, Wan, Wang, Wang, Wang, Wang, Xie, Xu, Xu, Xu, Yang, Yang, Yang, Yang, Yu, Zhang, Zhang, Zhang, Zheng, Zhong, Zhou, Zhou, Zhou, Zhu, and Zhu]{bai2025qwen3vl}
Shuai Bai, Yuxuan Cai, Ruizhe Chen, Keqin Chen, Xionghui Chen, Zesen Cheng, Lianghao Deng, Wei Ding, Chang Gao, Chunjiang Ge, Wenbin Ge, Zhifang Guo, Qidong Huang, Jie Huang, Fei Huang, Binyuan Hui, Shutong Jiang, Zhaohai Li, Mingsheng Li, Mei Li, Kaixin Li, Zicheng Lin, Junyang Lin, Xuejing Liu, Jiawei Liu, Chenglong Liu, Yang Liu, Dayiheng Liu, Shixuan Liu, Dunjie Lu, Ruilin Luo, Chenxu Lv, Rui Men, Lingchen Meng, Xuancheng Ren, Xingzhang Ren, Sibo Song, Yuchong Sun, Jun Tang, Jianhong Tu, Jianqiang Wan, Peng Wang, Pengfei Wang, Qiuyue Wang, Yuxuan Wang, Tianbao Xie, Yiheng Xu, Haiyang Xu, Jin Xu, Zhibo Yang, Mingkun Yang, Jianxin Yang, An~Yang, Bowen Yu, Fei Zhang, Hang Zhang, Xi~Zhang, Bo~Zheng, Humen Zhong, Jingren Zhou, Fan Zhou, Jing Zhou, Yuanzhi Zhu, and Ke~Zhu.
\newblock Qwen3-vl technical report.
\newblock \emph{arXiv preprint arXiv:2511.21631}, 2025.

\bibitem[Carlini et~al.(2023)Carlini, Hayes, Nasr, Jagielski, Sehwag, Tram{\`e}r, Balle, Ippolito, and Wallace]{carlini2023extracting}
Nicholas Carlini, Jamie Hayes, Milad Nasr, Matthew Jagielski, Vikash Sehwag, Florian Tram{\`e}r, Borja Balle, Daphne Ippolito, and Eric Wallace.
\newblock Extracting training data from diffusion models.
\newblock In \emph{32nd USENIX Security Symposium}, pp.\  5253--5270. USENIX Association, 2023.

\bibitem[Dai \& Gifford(2023)Dai and Gifford]{dai2023training}
Zheng Dai and David~K. Gifford.
\newblock Training data attribution for diffusion models.
\newblock \emph{arXiv preprint arXiv:2306.02174}, 2023.

\bibitem[Georgiev et~al.(2023)Georgiev, Vendrow, Salman, Park, and Madry]{georgiev2023journey}
Kristian Georgiev, Joshua Vendrow, Hadi Salman, Sung~Min Park, and Aleksander Madry.
\newblock The journey, not the destination: How data guides diffusion models.
\newblock In \emph{ICML Workshop on Challenges in Deployable Generative AI}, 2023.

\bibitem[Ghorbani \& Zou(2019)Ghorbani and Zou]{ghorbani2019shapley}
Amirata Ghorbani and James Zou.
\newblock Data shapley: Equitable valuation of data for machine learning.
\newblock In \emph{Proceedings of the 36th International Conference on Machine Learning}, volume~97 of \emph{Proceedings of Machine Learning Research}, pp.\  2242--2251. PMLR, 2019.

\bibitem[Ho et~al.(2020)Ho, Jain, and Abbeel]{ho2020ddpm}
Jonathan Ho, Ajay Jain, and Pieter Abbeel.
\newblock Denoising diffusion probabilistic models.
\newblock In \emph{Advances in Neural Information Processing Systems}, volume~33, pp.\  6840--6851, 2020.

\bibitem[Hu et~al.(2022)Hu, Shen, Wallis, Allen-Zhu, Li, Wang, Wang, and Chen]{hu2022lora}
Edward~J. Hu, Yelong Shen, Phillip Wallis, Zeyuan Allen-Zhu, Yuanzhi Li, Shean Wang, Lu~Wang, and Weizhu Chen.
\newblock {LoRA}: Low-rank adaptation of large language models.
\newblock In \emph{International Conference on Learning Representations}, 2022.

\bibitem[Ilyas et~al.(2022)Ilyas, Park, Engstrom, Leclerc, and Madry]{ilyas2022datamodels}
Andrew Ilyas, Sung~Min Park, Logan Engstrom, Guillaume Leclerc, and Aleksander Madry.
\newblock Datamodels: Understanding predictions with data and data with predictions.
\newblock In \emph{Proceedings of the 39th International Conference on Machine Learning}, volume 162 of \emph{Proceedings of Machine Learning Research}, pp.\  9525--9587. PMLR, 2022.

\bibitem[Koh \& Liang(2017)Koh and Liang]{koh2017influence}
Pang~Wei Koh and Percy Liang.
\newblock Understanding black-box predictions via influence functions.
\newblock In \emph{Proceedings of the 34th International Conference on Machine Learning}, volume~70 of \emph{Proceedings of Machine Learning Research}, pp.\  1885--1894. PMLR, 2017.

\bibitem[Krizhevsky(2009)]{krizhevsky2009cifar}
Alex Krizhevsky.
\newblock Learning multiple layers of features from tiny images.
\newblock Technical report, University of Toronto, 2009.

\bibitem[Liao et~al.(2022)Liao, Li, Liu, and Keutzer]{liao2022artbench}
Peiyuan Liao, Xiuyu Li, Xihui Liu, and Kurt Keutzer.
\newblock The artbench dataset: Benchmarking generative models with artworks.
\newblock \emph{arXiv preprint arXiv:2206.11404}, 2022.

\bibitem[Lin et~al.(2025)Lin, Tao, Dong, and Xu]{lin2025das}
Jinxu Lin, Linwei Tao, Minjing Dong, and Chang Xu.
\newblock Diffusion attribution score: Evaluating training data influence in diffusion models.
\newblock In \emph{International Conference on Learning Representations}, 2025.

\bibitem[Mlodozeniec et~al.(2024)Mlodozeniec, Eschenhagen, Bae, Immer, Krueger, and Turner]{mlodozeniec2024influence}
Bruno Mlodozeniec, Runa Eschenhagen, Juhan Bae, Alexander Immer, David Krueger, and Richard~E. Turner.
\newblock Influence functions for scalable data attribution in diffusion models.
\newblock \emph{arXiv preprint arXiv:2410.13850}, 2024.

\bibitem[Nguyen et~al.(2025)Nguyen, Le~Nguyen, Lucey, and Hoai]{nguyen2025region}
Trong~Bang Nguyen, Phi Le~Nguyen, Simon Lucey, and Minh Hoai.
\newblock Region-level data attribution for text-to-image generative models.
\newblock In \emph{Proceedings of the IEEE/CVF International Conference on Computer Vision}, pp.\  18825--18833, 2025.

\bibitem[Nichol \& Dhariwal(2021)Nichol and Dhariwal]{nichol2021improved}
Alexander~Quinn Nichol and Prafulla Dhariwal.
\newblock Improved denoising diffusion probabilistic models.
\newblock In \emph{Proceedings of the 38th International Conference on Machine Learning}, volume 139 of \emph{Proceedings of Machine Learning Research}, pp.\  8162--8171. PMLR, 2021.

\bibitem[Park et~al.(2023)Park, Georgiev, Ilyas, Leclerc, and Madry]{park2023trak}
Sung~Min Park, Kristian Georgiev, Andrew Ilyas, Guillaume Leclerc, and Aleksander Madry.
\newblock {TRAK}: Attributing model behavior at scale.
\newblock In \emph{Proceedings of the 40th International Conference on Machine Learning}, volume 202 of \emph{Proceedings of Machine Learning Research}, pp.\  27074--27113. PMLR, 2023.

\bibitem[Park et~al.(2026)Park, Lai, Hayakawa, Takida, Murata, Liao, Choi, Cheuk, Koo, and Mitsufuji]{park2025concept}
Yonghyun Park, Chieh-Hsin Lai, Satoshi Hayakawa, Yuhta Takida, Naoki Murata, Wei-Hsiang Liao, Woosung Choi, Kin~Wai Cheuk, Junghyun Koo, and Yuki Mitsufuji.
\newblock {Concept-TRAK}: Understanding how diffusion models learn concepts through concept-level attribution.
\newblock In \emph{International Conference on Learning Representations}, 2026.
\newblock URL \url{https://arxiv.org/abs/2507.06547}.

\bibitem[Pruthi et~al.(2020)Pruthi, Liu, Kale, and Sundararajan]{pruthi2020tracin}
Garima Pruthi, Frederick Liu, Satyen Kale, and Mukund Sundararajan.
\newblock Estimating training data influence by tracing gradient descent.
\newblock In \emph{Advances in Neural Information Processing Systems}, volume~33, pp.\  19920--19930, 2020.

\bibitem[Rombach et~al.(2022)Rombach, Blattmann, Lorenz, Esser, and Ommer]{rombach2022high}
Robin Rombach, Andreas Blattmann, Dominik Lorenz, Patrick Esser, and Bjorn Ommer.
\newblock High-resolution image synthesis with latent diffusion models.
\newblock In \emph{Proceedings of the IEEE/CVF Conference on Computer Vision and Pattern Recognition}, pp.\  10684--10695, 2022.
\newblock \doi{10.1109/CVPR52688.2022.01042}.

\bibitem[Sohl-Dickstein et~al.(2015)Sohl-Dickstein, Weiss, Maheswaranathan, and Ganguli]{sohl2015nonequilibrium}
Jascha Sohl-Dickstein, Eric Weiss, Niru Maheswaranathan, and Surya Ganguli.
\newblock Deep unsupervised learning using nonequilibrium thermodynamics.
\newblock In \emph{Proceedings of the 32nd International Conference on Machine Learning}, volume~37 of \emph{Proceedings of Machine Learning Research}, pp.\  2256--2265. PMLR, 2015.

\bibitem[Somepalli et~al.(2023)Somepalli, Singla, Goldblum, Geiping, and Goldstein]{somepalli2023replication}
Gowthami Somepalli, Vasu Singla, Micah Goldblum, Jonas Geiping, and Tom Goldstein.
\newblock Diffusion art or digital forgery? investigating data replication in diffusion models.
\newblock In \emph{Proceedings of the IEEE/CVF Conference on Computer Vision and Pattern Recognition}, pp.\  6048--6058, 2023.

\bibitem[Song et~al.(2021{\natexlab{a}})Song, Meng, and Ermon]{song2021ddim}
Jiaming Song, Chenlin Meng, and Stefano Ermon.
\newblock Denoising diffusion implicit models.
\newblock In \emph{International Conference on Learning Representations}, 2021{\natexlab{a}}.

\bibitem[Song et~al.(2021{\natexlab{b}})Song, Sohl-Dickstein, Kingma, Kumar, Ermon, and Poole]{song2021score}
Yang Song, Jascha Sohl-Dickstein, Diederik~P. Kingma, Abhishek Kumar, Stefano Ermon, and Ben Poole.
\newblock Score-based generative modeling through stochastic differential equations.
\newblock In \emph{International Conference on Learning Representations}, 2021{\natexlab{b}}.

\bibitem[Tang et~al.(2023)Tang, Liu, Pandey, Jiang, Yang, Kumar, Stenetorp, Lin, and Ture]{tang2023daam}
Raphael Tang, Linqing Liu, Akshat Pandey, Zhiying Jiang, Gefei Yang, Karun Kumar, Pontus Stenetorp, Jimmy Lin, and Ferhan Ture.
\newblock What the {DAAM}: Interpreting stable diffusion using cross attention.
\newblock In \emph{Proceedings of the 61st Annual Meeting of the Association for Computational Linguistics}, pp.\  5644--5659. Association for Computational Linguistics, 2023.
\newblock \doi{10.18653/v1/2023.acl-long.310}.

\bibitem[Xie et~al.(2024)Xie, Li, Bai, and Hsieh]{xie2024diffusionretrac}
Tong Xie, Haoyu Li, Andrew Bai, and Cho-Jui Hsieh.
\newblock Data attribution for diffusion models: Timestep-induced bias in influence estimation.
\newblock \emph{arXiv preprint arXiv:2401.09031}, 2024.

\bibitem[Yeh et~al.(2018)Yeh, Kim, Yen, and Ravikumar]{yeh2018representer}
Chih-Kuan Yeh, Joon~Sik Kim, Ian E.~H. Yen, and Pradeep Ravikumar.
\newblock Representer point selection for explaining deep neural networks.
\newblock In \emph{Advances in Neural Information Processing Systems}, volume~31, 2018.

\bibitem[Zheng et~al.(2024)Zheng, Pang, Du, Jiang, and Lin]{zheng2024dtrak}
Xiaosen Zheng, Tianyu Pang, Chao Du, Jing Jiang, and Min Lin.
\newblock Intriguing properties of data attribution on diffusion models.
\newblock In \emph{International Conference on Learning Representations}, 2024.

\end{thebibliography}
\bibliographystyle{iclr2027_conference}

\appendix

\setcounter{figure}{0}
\renewcommand{\thefigure}{A\arabic{figure}}

\setcounter{table}{0}
\renewcommand{\thetable}{A\arabic{table}}

\setcounter{equation}{0}
\renewcommand{\theequation}{A\arabic{equation}}

\setcounter{theorem}{0}
\renewcommand{\thetheorem}{\arabic{theorem}}

\section*{Appendix}\label{appendix}

\section{Dynamic Attribution Signature}
\label{app:dynamic-descriptor}

This section provides the mathematical definitions of the dynamic features used
to construct the trajectory representation in Section~\ref{sec:dynamic-signature}.
These features characterize complementary aspects of the factor-specific
counterfactual response, including response strength, first-order directional
evolution, and trajectory reliability.

For convenience, we denote the counterfactual response at layer $\ell$,
timestep $t_j$, and analysis probe $s$ as: 
\begin{equation}
R_{s,\ell,t_j}(x)
=
\mathcal R_{c,s}^{\ell}(x,t_j)
\in
\mathbb R^{P_\ell\times C_\ell},
\end{equation}
where $P_\ell$ and $C_\ell$ denote the token and channel dimensions,
respectively.

\paragraph{Response Energy.}
The response energy measures the strength of the condition-induced activation.
For channel $k$, we compute the spatial RMS magnitude as: 
\begin{equation}
E_{s,\ell,t_j,k}(x)
=
\log
\left(
\sqrt{
\frac{1}{P_\ell}
\sum_{p=1}^{P_\ell}
R_{s,\ell,t_j,p,k}(x)^2
}
+\epsilon
\right).
\label{eq:app-energy}
\end{equation}
This quantity captures how strongly the target factor affects each channel at a given layer and diffusion state. Energy features are standardized using statistics estimated exclusively from the training bank.

\paragraph{First-Order Directional Dynamics.}
Let $u_j=\log\operatorname{SNR}(t_j)$. We define the first-order response velocity between adjacent recorded states as: 
\begin{equation}
V_{s,\ell,t_j}(x)
=
\frac{
R_{s,\ell,t_{j+1}}(x)-R_{s,\ell,t_j}(x)
}{
u_{j+1}-u_j
}.
\label{eq:app-first-order}
\end{equation}
Unlike response energy, which describes activation strength, $V_{s,\ell,t_j}$ characterizes how the factor-specific response evolves along the denoising trajectory. For each channel, we normalize the spatial velocity and retain its direction: 
\begin{equation}
d_{s,\ell,t_j,k}(x)
=
\frac{
V_{s,\ell,t_j,:,k}(x)
}{
\|V_{s,\ell,t_j,:,k}(x)\|_2+\epsilon
},
\label{eq:app-direction}
\end{equation}
where $V_{s,\ell,t_j,:,k}$ denotes the velocity vector of channel $k$
over all token positions.

\paragraph{Multi-Probe Directional Stability.}
Repeated probes are used to determine whether the observed dynamic direction is stable under different analysis noise realizations. Let
\begin{equation}
m^V_{\ell,t_j,k}(x)
=
\frac{1}{S}
\sum_{s=1}^{S}
d_{s,\ell,t_j,k}(x).
\end{equation}
We define the directional coherence as: 
\begin{equation}
\rho_{\ell,t_j,k}(x)
=
\operatorname{clip}_{[0,1]}
\left(
\|m^V_{\ell,t_j,k}(x)\|_2
\right),
\label{eq:app-direction-stability}
\end{equation}
and the corresponding mean direction as: 
\begin{equation}
D_{\ell,t_j,k}(x)
=
\frac{
m^V_{\ell,t_j,k}(x)
}{
\|m^V_{\ell,t_j,k}(x)\|_2+\epsilon
}.
\label{eq:app-mean-direction}
\end{equation}
A large $\rho_{\ell,t_j,k}$ indicates that the response evolves consistently across probes, whereas a small value indicates unstable or conflicting directions.

\paragraph{Energy Stability.}
We similarly measure whether the response-strength pattern is consistent across repeated probes. Let $\widetilde E_{s,\ell,t_j}(x)$ denote the standardized channel-wise energy vector. Its cross-probe agreement is
\begin{equation}
r^E_{\ell,t_j}(x)
=
\operatorname{clip}_{[0,1]}
\left[
\frac{2}{S(S-1)}
\sum_{s<s'}
\cos
\left(
\widetilde E_{s,\ell,t_j}(x),
\widetilde E_{s',\ell,t_j}(x)
\right)
\right].
\label{eq:app-energy-stability}
\end{equation}
The reliability-weighted energy representation is therefore:
\begin{equation}
\psi^E_{\ell,t_j}(x)
=
r^E_{\ell,t_j}(x)
\,
\overline E_{\ell,t_j}(x),
\label{eq:app-energy-block}
\end{equation}
where $\overline E_{\ell,t_j}$ denotes the energy vector averaged across probes.

\paragraph{Second-Order Reliability.}
Second-order dynamics are used only to assess the reliability of first-order evolution. We define:
\begin{equation}
A_{s,\ell,t_j,k}(x)
=
\frac{
2\left(
V_{s,\ell,t_{j+1},:,k}(x)
-
V_{s,\ell,t_j,:,k}(x)
\right)
}{
u_{j+2}-u_j
}.
\label{eq:app-second-order}
\end{equation}
Let $\overline A_{\ell,t_j,k}$ denote its mean across probes. The corresponding second-order reliability is
\begin{equation}
r^A_{\ell,t_j,k}(x)
=
\operatorname{clip}_{[0,1]}
\left[
\left(
1+
\frac{
\frac{1}{S}
\sum_s
\|A_{s,\ell,t_j,k}(x)-\overline A_{\ell,t_j,k}(x)\|_2^2
}{
\|\overline A_{\ell,t_j,k}(x)\|_2^2+\epsilon
}
\right)^{-1}
\right].
\label{eq:app-second-order-reliability}
\end{equation}
A high value indicates consistent local acceleration across probes, while a low value suppresses first-order directions whose local evolution is unstable.

Combining directional coherence and second-order reliability gives
\begin{equation}
\psi^V_{\ell,t_j,k}(x)
=
\sqrt{
\rho_{\ell,t_j,k}(x)
\widetilde r^A_{\ell,t_j,k}(x)
}
\,D_{\ell,t_j,k}(x),
\label{eq:app-velocity-block}
\end{equation}
where $\widetilde r^A_{\ell,t_j,k}$ denotes the second-order reliability aligned with velocity interval $t_j$.

\section{Theoretical Analysis}
\label{app:theory}

All results in this section concern the frozen dynamic representation used during attribution. Our analysis instead asks a representation-level question: once the dynamic descriptors and training-bank statistics have been fixed, what is the mathematical meaning of the bank-adaptive interaction used by \method{}? We show that this interaction is exactly a ridge-reconstruction coefficient and that it governs an exact leave-one-descriptor-out quantity.

For completeness, since $Z_b^\top Z_b+\lambda I\succ0$ for $\lambda>0$ and $w_b\ge0$, the kernel $K_\lambda$ in Eq.~\eqref{eq:rda-kernel} is positive semidefinite.

\subsection{Ridge Reconstruction Interpretation}

Consider a dense descriptor bank: 
\begin{equation}
Z
=
\begin{bmatrix}
z_1^\top\\
\vdots\\
z_N^\top
\end{bmatrix}
\in\mathbb R^{N\times d},
\end{equation}
and define \(A=Z^\top Z+\lambda I_d\) with \(\lambda>0\).

We first show that the bank-adaptive interaction has a direct reconstruction
interpretation.

\begin{proposition}[Ridge reconstruction coefficient]
\label{prop:ridge-coefficient}

Consider the regularized reconstruction problem: 
\begin{equation}
a^\star
=
\arg\min_{a\in\mathbb R^N}
\|z_q-Z^\top a\|_2^2
+
\lambda\|a\|_2^2.
\label{eq:app-ridge-reconstruction}
\end{equation}
Then: 
\begin{equation}
a^\star
=
(ZZ^\top+\lambda I_N)^{-1}Zz_q
=
Z(Z^\top Z+\lambda I_d)^{-1}z_q.
\label{eq:app-ridge-solution}
\end{equation}
Consequently, the coefficient associated with training descriptor $z_i$ is: 
\begin{equation}
\boxed{
a_i^\star
=
z_q^\top A^{-1}z_i
}.
\label{eq:app-ridge-coefficient}
\end{equation}
\end{proposition}

\begin{proof}
The normal equations of Eq.~\eqref{eq:app-ridge-reconstruction} give
\((ZZ^\top+\lambda I_N)a^\star=Zz_q\).
Moreover,
\((ZZ^\top+\lambda I_N)Z=Z(Z^\top Z+\lambda I_d)\),
and therefore
\((ZZ^\top+\lambda I_N)^{-1}Z
=Z(Z^\top Z+\lambda I_d)^{-1}\).
Substitution yields Eq.~\eqref{eq:app-ridge-solution}, whose $i$-th entry gives Eq.~\eqref{eq:app-ridge-coefficient}.
\end{proof}

Thus, the interaction used by \method{} is not merely a covariance-adjusted similarity. In the dense representation, it is exactly the coefficient assigned to a training descriptor when reconstructing the query under ridge regularization.

\subsection{Exact Representation-Level Counterfactual Removal}

The reconstruction interpretation above also yields an exact representation-level counterfactual quantity. We ask how much the optimal ridge reconstruction deteriorates when one training descriptor is removed while the descriptor map and all fitted transformations remain fixed.

Define the optimal reconstruction objective as:
\begin{equation}
E_Z(z_q)
=
\min_a
\|z_q-Z^\top a\|_2^2
+
\lambda\|a\|_2^2.
\label{eq:app-ridge-objective}
\end{equation}

\begin{theorem}[Exact leave-one-descriptor-out effect]
\label{thm:loo}

Let \(A=Z^\top Z+\lambda I\). Then:
\begin{equation}
E_Z(z_q)
=
\lambda z_q^\top A^{-1}z_q.
\label{eq:app-ridge-optimum}
\end{equation}

For training descriptor $z_i$, define
\(h_i=z_i^\top A^{-1}z_i\) and
\(k_i(z_q)=z_q^\top A^{-1}z_i\).
Then \(0\le h_i<1\), and removing $z_i$ increases the optimal reconstruction objective by: 
\begin{equation}
\boxed{
E_{Z_{-i}}(z_q)-E_Z(z_q)
=
\lambda
\frac{
k_i(z_q)^2
}{
1-h_i
}
}
\ge0.
\label{eq:app-loo}
\end{equation}
\end{theorem}

\begin{proof}
Let
\(C=Z_{-i}^\top Z_{-i}+\lambda I\succ0\),
so that \(A=C+z_iz_i^\top\).

Writing
\(t=z_i^\top C^{-1}z_i\ge0\),
the Sherman--Morrison identity gives: 
\begin{equation}
A^{-1}
=
C^{-1}
-
\frac{
C^{-1}z_iz_i^\top C^{-1}
}{
1+t
}.
\end{equation}
Hence \(h_i=t/(1+t)<1\).

Equivalently, the rank-one downdate is: 
\begin{equation}
C^{-1}
=
A^{-1}
+
\frac{
A^{-1}z_iz_i^\top A^{-1}
}{
1-h_i
}.
\end{equation}

Using Eq.~\eqref{eq:app-ridge-optimum}: 
\begin{align}
E_{Z_{-i}}(z_q)-E_Z(z_q)
&=
\lambda
z_q^\top
(C^{-1}-A^{-1})
z_q
\\
&=
\lambda
\frac{
\left(
z_q^\top A^{-1}z_i
\right)^2
}{
1-h_i
},
\end{align}
which proves Eq.~\eqref{eq:app-loo}.
\end{proof}

The theorem identifies
\(k_i(z_q)=z_q^\top A^{-1}z_i\)
as the primitive cross-interaction governing exact descriptor removal. The leave-one-descriptor-out effect is unsigned because it depends on $k_i(z_q)^2$ together with the leverage correction $(1-h_i)^{-1}$.

For factor-specific attribution, however, the sign of the interaction is informative: positive and negative values distinguish aligned from opposing dynamic responses in the bank-adaptive geometry. We therefore use the signed interaction for ranking and reserve the unsigned removal magnitude for counterfactual interpretation.The theorem therefore provides a counterfactual interpretation of the interaction underlying the attribution score, while the score itself remains distinct from the leave-one-descriptor-out quantity.

Importantly, this counterfactual is defined entirely in the frozen descriptor space. It does not imply that descriptor removal is equivalent to deleting the corresponding training example and retraining the diffusion model.

\subsection{Blockwise Extension}

The implemented \method{} geometry uses multiple local descriptor blocks rather than a single dense concatenated representation. We therefore define the corresponding block-separable reconstruction objective: 
\begin{equation}
E_{\mathrm{blk}}(q)
=
\sum_b
w_b
\min_{a_b\in\mathbb R^N}
\left[
\|h_b(q)-Z_b^\top a_b\|_2^2
+
\lambda\|a_b\|_2^2
\right].
\label{eq:app-block-reconstruction}
\end{equation}

For each block, recall
\(G_b=(Z_b^\top Z_b+\lambda I)^{-1}\),
and define
\(h_{ib}=h_b(x_i)^\top G_bh_b(x_i)\) and
\(k_{ib}(q)=h_b(q)^\top G_bh_b(x_i)\).

Applying Theorem~\ref{thm:loo} independently to each block gives:
\begin{equation}
\boxed{
E_{\mathrm{blk},-i}(q)
-
E_{\mathrm{blk}}(q)
=
\lambda
\sum_b
w_b
\frac{
k_{ib}(q)^2
}{
1-h_{ib}
}
}
\ge0.
\label{eq:app-block-removal}
\end{equation}

The blockwise removal quantity is not obtained by squaring the final score $K_\lambda(q,x_i)$, since each block receives its own leverage correction before aggregation. Nevertheless, both quantities are constructed from the same signed bank-adaptive interactions $k_{ib}(q)$. The theoretical connection is therefore at the level of these local cross-interactions.

As in the dense case, this result concerns only the frozen representation space and should not be interpreted as model-retraining influence.

\subsection{Stability of the Training-Bank Geometry}

Finally, we examine the sensitivity of the fitted inverse geometry while holding query and candidate descriptors fixed. This result concerns perturbations of the empirical second-order structure of the training bank; it does not characterize perturbations of the diffusion model or of the descriptor extraction process.

\begin{theorem}[Training-bank geometry stability]
\label{thm:stability}

Let \(A=C+\lambda I\) and
\(\widetilde A=C+E+\lambda I\), where
\(C\succeq0\), \(E=E^\top\), and
\(\delta=\|E\|_2<\lambda\).
Then, for fixed descriptors $z,z'$:
\begin{equation}
\left|
z^\top A^{-1}z'
-
z^\top\widetilde A^{-1}z'
\right|
\le
\|z\|_2
\|z'\|_2
\frac{
\delta
}{
\lambda(\lambda-\delta)
}.
\label{eq:app-stability}
\end{equation}
\end{theorem}

\begin{proof}
Since $C\succeq0$,
\(\|A^{-1}\|_2\le 1/\lambda\).
By Weyl's inequality,
\(\lambda_{\min}(\widetilde A)\ge\lambda-\delta\),
and therefore
\(\|\widetilde A^{-1}\|_2\le1/(\lambda-\delta)\).

Using the resolvent identity
\(\widetilde A^{-1}-A^{-1}
=-A^{-1}E\widetilde A^{-1}\),
we obtain: 
\[
\|
\widetilde A^{-1}-A^{-1}
\|_2
\le
\frac{
\delta
}{
\lambda(\lambda-\delta)
}.
\]
Finally,
\(|z^\top Bz'|
\le
\|z\|_2\|B\|_2\|z'\|_2\)
gives Eq.~\eqref{eq:app-stability}.
\end{proof}

Thus, within the frozen-descriptor setting, the ridge parameter controls the sensitivity of the bank-adaptive interaction to perturbations of the estimated second-order training-bank structure.

\section{Computational Complexity}
\label{app:complexity}

Let $D=\sum_b d_b$ denote the total descriptor dimension. The blockwise formulation requires
$O(N\sum_b d_b^2+\sum_b d_b^3)$ preprocessing time and
$O(\sum_b d_b^2)$ factor memory, compared with
$O(ND^2+D^3)$ time and $O(D^2)$ memory for a dense geometry over the concatenated descriptor. Once the descriptors and block factorizations are cached, scoring $Q$ queries against $N$ training examples costs $O(QND)$.

\section{Limitations and Future Work}
\label{sec:limitations}

A central limitation of \method{} is that its attribution signal is defined in the space of internal trajectory representations rather than directly through training-induced parameter updates. Although our reconstruction analysis provides an exact counterfactual interpretation in the frozen descriptor space, descriptor removal is not equivalent to removing a training example and retraining the diffusion model. A promising direction is therefore to connect multi-order trajectory dynamics with parameter gradients or local training perturbations, providing a tighter link between dynamic representations and model-level counterfactual influence.

Our ArtBench setting introduces an additional limitation. Stable Diffusion v1.4 is adapted with LoRA  \citep{hu2022lora} on a relatively small style-oriented dataset, and object-level attribution is generally weaker than style attribution. This may reflect both the strong object semantics already encoded during pretraining and the limited frequency of individual objects in the fine-tuning data. Future work could better separate attribution to pretrained knowledge from attribution introduced during downstream adaptation.

\section{Implementation and Evaluation Details}
\label{app:protocol}

\subsection{Why not LDS?}
Linear Datamodeling Score (LDS) measures whether additive attribution scores predict changes in a scalar utility across models trained on random data subsets \citep{park2023trak}. This is informative when the target is an image-level utility, but our evaluation asks which training examples support a specified factor in a generated image. As also noted by Concept-TRAK \citep{park2025concept}, the standard subset construction can obscure this signal when the factor is common: a reduced training set may still contain enough positive examples to preserve the concept, so utility changes remain small until the available support crosses a threshold. Such behavior is poorly summarized by a linear per-example response and can understate the relevance of concept-bearing samples.

LDS additionally requires training many subset models, which is particularly costly for text-to-image diffusion. We instead use retrieval precision with controlled factor labels or a frozen presence judge, directly testing whether the highest-ranked candidates contain the requested factor under shared query masks and candidate order. This is a task-specific choice: LDS remains appropriate for scalar-utility questions, while model-level subset retraining would provide a useful complementary validation. Our exact removal result in Theorem~\ref{thm:loo} concerns the frozen descriptor surrogate and is not presented as a replacement for such retraining.

\paragraph{Synthetic.}
Synthetic contains 11,024 $32\!\times\!32$ RGB images combining two shapes (circle and triangle) and two colors (red and blue), split into 10,000/512/512 training/validation/test images. Training and validation include blue triangles, blue circles, and red circles, with red triangles held out as compositional OOD. The temporal-support study uses an additional 500-query selection split and a disjoint 2,000-query main split, balanced between red triangles and blue circles and filtered by frozen shape and color classifiers. It evaluates eight 100-timestep intervals on the coarse $900\!\rightarrow\!100$ grid, all seven contiguous two-timestep windows, and all ten single-layer candidates.

\paragraph{CIFAR-100.}
We retain 59 fine classes (45 animal and 14 plant) and their original superclasses from $32\!\times\!32$ CIFAR-100. A fixed shuffle divides the 29,500 eligible training-split images into 26,500 attribution candidates and 3,000 validation images. The main bank contains 370 generated queries; an independently generated 171-query bank selects baseline regularization. A frozen multi-head Wide-ResNet-50-2 retains a query only when its animal/plant prediction matches the requested condition and its kingdom, superclass, and fine-class confidences all exceed $0.8$. The predicted superclass and fine class provide the hierarchical evaluation labels, and all methods share the accepted images and order.

\paragraph{ArtBench.}
ArtBench2 contains ukiyo-e and post-impressionism; ArtBench5 adds romanticism, Renaissance, and Baroque. After center-cropping to $256\!\times\!256$, the fixed subsets contain 6,000/13,500 real images: 5,000/12,500 for training and 1,000 for validation. The full-data LoRA generates 1,000 pure-style queries (500/200 per style), of which a frozen Qwen3-VL-8B-Instruct judge \citep{bai2025qwen3vl}  accepts 956/981. The object--style protocol generates 32/80 cat-or-dog queries (eight seeds per object--style pair), retaining 11/30 that contain both requested factors. Official metadata supplies style relevance; Qwen supplies query style/object validity and training-image object relevance. \method{} uses the bank-adaptive kernel $K_\lambda$ with $\lambda=0.01N$ on both subsets. For DAS, $\lambda$ is selected by mean generated-query LDS on the first 64 official 50\% subsets from 27 candidates spanning $10^{-2}$ to $5\!\times\!10^6$; neither P@$K$ nor composition labels enter this selection.

\section{Alternative Score Designs}
\label{app:score-designs}

\method{} ranks training examples with the signed bank-adaptive kernel
$K_\lambda$. The implementation also exposes two diagnostic alternatives using
the same fixed standardized block descriptors, with
$h(x)=\bigoplus_b\sqrt{w_b}h_b(x)$:
\begin{align}
K_\lambda(x,x')
&=
\sum_b w_b h_b(x)^\top G_b h_b(x'),
&&\text{\method{} bank-adaptive kernel},\\
C(x,x')
&=
\frac{h(x)^\top h(x')}
{\|h(x)\|_2\|h(x')\|_2},
&&\text{ordinary-cosine variant},\\
S_\lambda(x,x')
&=
\frac{K_\lambda(x,x')}
{\sqrt{K_\lambda(x,x)K_\lambda(x',x')}},
&&\text{normalized ridge variant}.
\end{align}
Cosine ignores the covariance structure of the training bank. $K_\lambda$ retains the signed ridge-whitened magnitude--direction interaction and, in the dense single-block case, corresponds directly to the interaction term arising from the reconstruction and descriptor-removal analysis. $S_\lambda$ uses the same bank-adaptive geometry but normalizes the ridge-whitened norms. We use $K_\lambda$ as the attribution score of \method{} throughout the main paper, while reporting $S_\lambda$ and ordinary cosine as exploratory score variants. Table~\ref{tab:label-flip-score-variants} compares the three scores using the identical full late-support descriptor, query bank, and candidate order; only the scoring function changes.

\begin{table*}[h]
\centering
\caption{
Exploratory score-design comparison for CIFAR-100 label-flip recovery. Values are P@$1/5/10/50$. All rows use the identical full late-support descriptor, query bank, and candidate order; only the scoring function changes.}
\label{tab:label-flip-score-variants}
\small
\resizebox{\textwidth}{!}{%
\begin{tabular}{lccc}
\toprule
Score
& Macro
& Flipped animal $\rightarrow$ plant
& Genuine plant \\
\midrule
\method{} ($K_\lambda$)
& 67.57 / 66.86 / 66.51 / 64.48
& 37.84 / 38.27 / 37.51 / 34.42
& 97.30 / 95.46 / 95.51 / 94.53 \\
Normalized variant ($S_\lambda$)
& 70.54 / 68.27 / 67.35 / 65.32
& 44.32 / 40.86 / 39.78 / 36.08
& 96.76 / 95.68 / 94.92 / 94.56 \\
Ordinary cosine
& 74.59 / 73.41 / 72.19 / 70.58
& 53.51 / 52.54 / 50.70 / 48.04
& 95.68 / 94.27 / 93.68 / 93.11 \\
\bottomrule
\end{tabular}}
\end{table*}

\section{Layer--Timestep Diagnostics}
\label{app:k-diagnostic}

The main experiments show that \method{} benefits from modeling factor-specific
response trajectories. Here, we further examine where this attribution evidence
arises within the network and how it evolves over denoising time. After the
factor-specific response is obtained at each denoising stage, \method{} preserves
these responses in temporal order and characterizes their evolution through
direction and energy. Direction captures signed agreement in response evolution,
whereas energy reflects response magnitude. We first analyze these quantities in
the selected low-noise region and then extend the analysis over the full
from (t=900) to (t=100) trajectory.

\begin{figure*}[h]
  \centering
  \includegraphics[width=\textwidth]
  {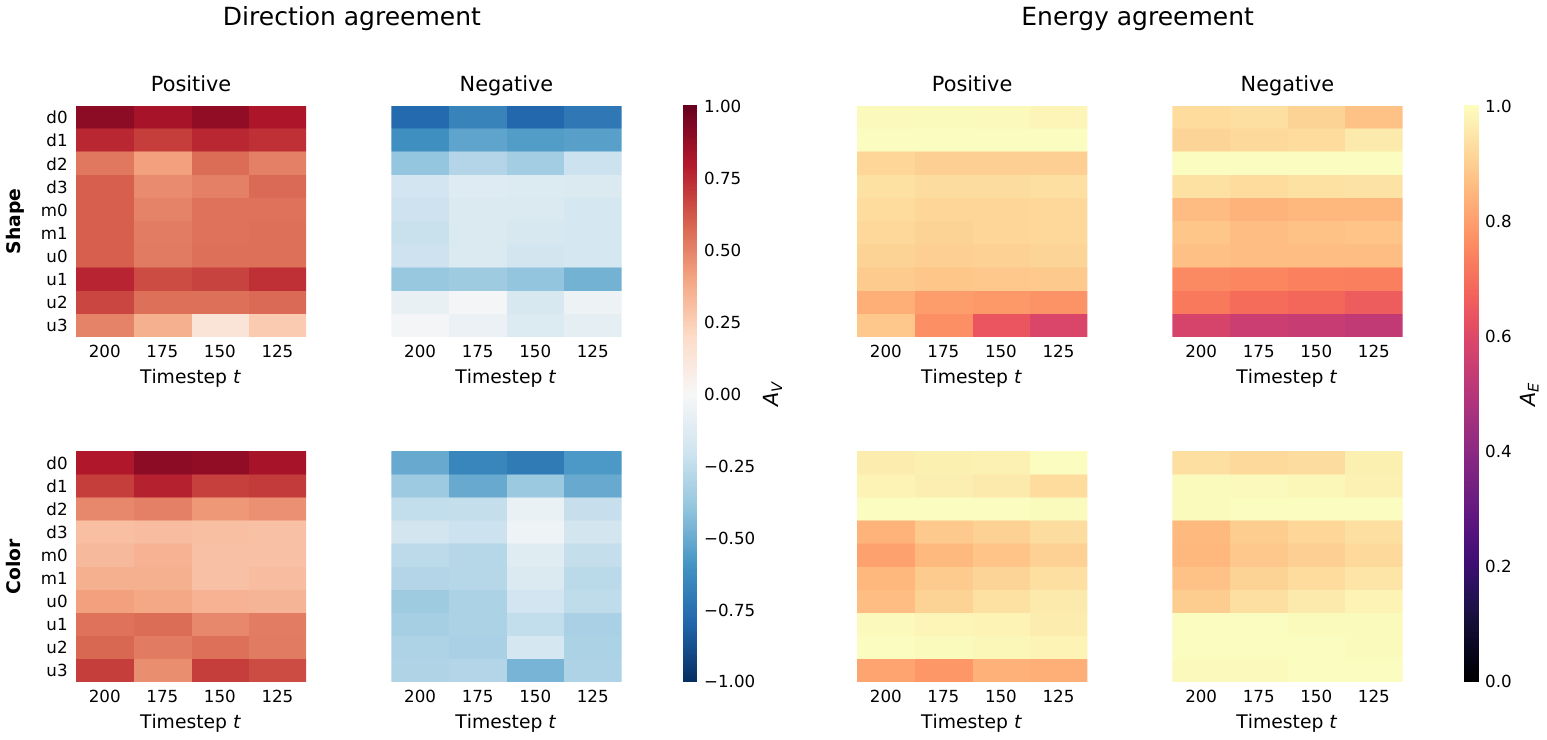}
  \caption{
  Fine-grained direction and energy patterns in the selected low-noise region.
  Representative positive and negative pairs are selected according to their
  separation under $K_\lambda$.
  }
  \label{fig:synthetic-k-pair-geometry}
\end{figure*}

Figure~\ref{fig:synthetic-k-pair-geometry} provides a local view of the dynamic
signals used by \method{}. Positive examples show consistently stronger
directional agreement with the query across layers and timesteps, whereas
negative examples exhibit weaker or opposing responses. Energy captures a
different aspect of the same trajectories by describing how strongly the
factor-specific response is expressed. These complementary patterns explain why
the dynamic descriptor benefits from retaining both the evolution direction and
response magnitude rather than reducing each example to a single static feature.

\begin{figure*}[h]
  \centering
  \includegraphics[width=\textwidth]
  {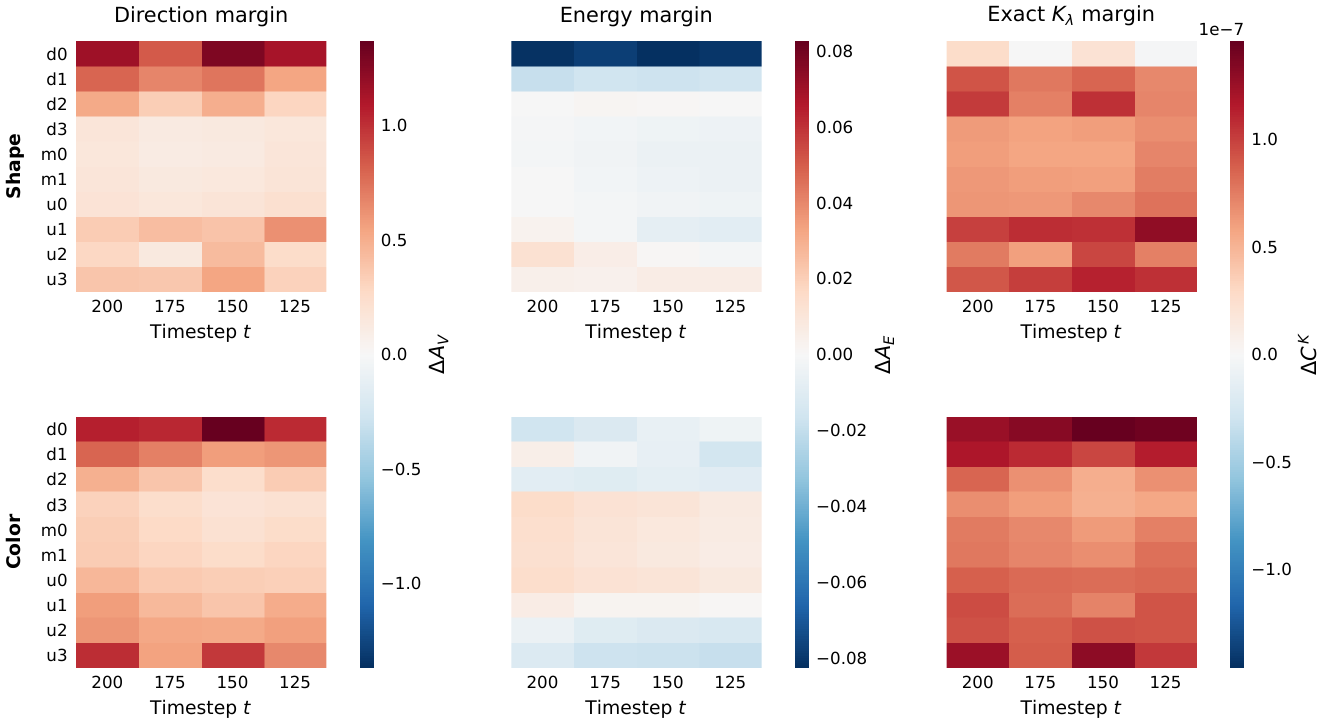}
  \caption{
  Matched-minus-mismatched direction, energy, and $K_\lambda$ cell margins
  in the fine-grained $200\!\rightarrow\!100$ analysis.
  }
  \label{fig:synthetic-k-alignment}
\end{figure*}

Figure~\ref{fig:synthetic-k-alignment} examines how these local dynamic
components relate to the final bank-adaptive interaction. For color, the
direction margin closely follows the $K_\lambda$ margin across the
layer--timestep grid, while energy exhibits a distinct pattern. Shape shows a
weaker but still structured correspondence. This indicates that $K_\lambda$
does not rely on response magnitude alone: directional evolution provides a
strong discriminative signal, while energy contributes complementary information
about where the response is expressed.

We next extend the analysis to the full from (t=900) to (t=100) trajectory.
Because the same factor-specific response is extracted at successive denoising
stages, these local responses can be organized into an ordered trajectory rather
than treated as independent timestep features. We aggregate layer--timestep
contributions into non-overlapping 100-timestep intervals to examine where the
resulting attribution support is concentrated across network depth and time.
A held-out selection split is used to identify representative layers and
temporal windows, while a disjoint main split is used for the final comparison.

\begin{table*}[h]
\centering
\caption{
Full-range layer--time analysis under $K_\lambda$. Correlations summarize the
correspondence of direction and energy with $K_\lambda$, while the final columns
report the selected layer and temporal window together with their separation
from the strongest alternative.
}
\label{tab:synthetic-full-grid}
\small
\resizebox{\textwidth}{!}{%
\begin{tabular}{lrrcc}
\toprule
Target
& Direction--$K_\lambda$ $\rho$
& Energy--$K_\lambda$ $\rho$
& Selected layer: gap [95\% CI] ($10^{-8}$)
& Selected window: gap [95\% CI] ($10^{-6}$) \\
\midrule
Shape
& $0.728$
& $-0.109$
& d1: $2.01\ [1.36,2.65]$
& $300\!\rightarrow\!100$: $0.871\ [0.852,0.890]$ \\
Color
& $0.972$
& $0.025$
& d0: $9.80\ [9.69,9.91]$
& $300\!\rightarrow\!100$: $1.097\ [1.078,1.117]$ \\
\bottomrule
\end{tabular}}
\end{table*}

Table~\ref{tab:synthetic-full-grid} confirms that directional evolution remains
strongly associated with $K_\lambda$ support over the full trajectory,
especially for color, whereas energy shows little monotone correspondence.
The strongest layer also depends on the target factor, with d1 selected for
shape and d0 for color, showing that factor-specific evidence is not localized
to a common network depth. In contrast, both factors concentrate their strongest
temporal support in the from (t=300) to (t=100)region. Together, these results
show that \method{} benefits from tracing the same factor-specific response
across denoising stages and aggregating the informative portions of its
layer--time trajectory.

\section{ArtBench Details and Visualizations}
\label{app:artbench}

ArtBench provides style labels but no ground-truth object annotations. We therefore freeze a
vision--language judge to determine object presence once and apply the same
valid-query mask to every method, preventing evaluation differences from being
introduced by query selection. Figure~\ref{fig:artbench-style-samples} shows
representative samples from the fixed ArtBench5 query bank, while
Figure~\ref{fig:artbench-composition-grid} visualizes the complete cat--style
composition bank before validity filtering and documents the queries entering
the composition evaluation. Figure~\ref{fig:artbench-retrieval} further provides
representative style-retrieval results on ArtBench2 and ArtBench5, illustrating
the training examples returned for generated queries under the corresponding
target styles.

\begin{figure*}[h]
  \centering
  \includegraphics[width=.98\textwidth]{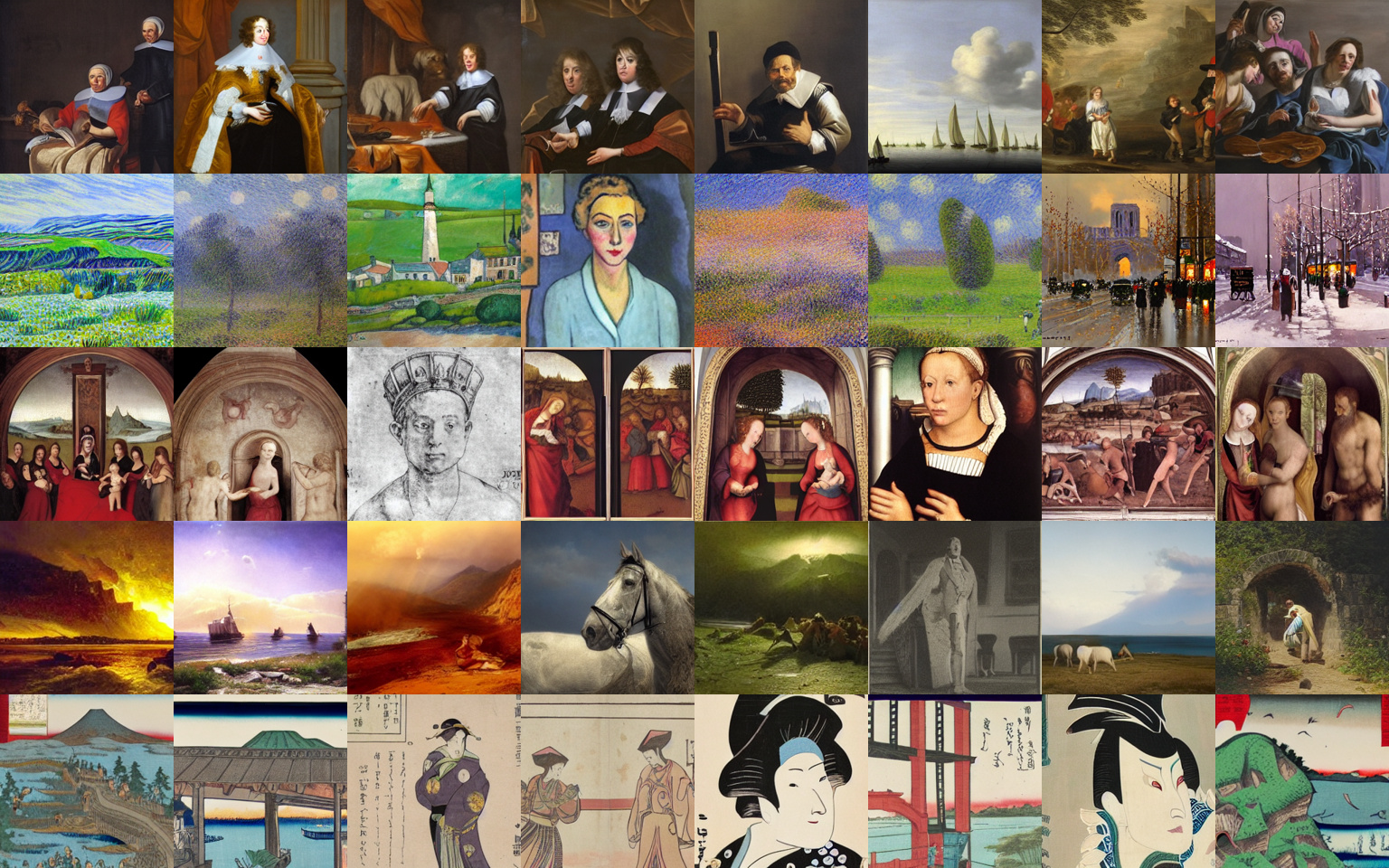}
  \caption{
  Examples from the fixed generated-query bank for ArtBench5.
  The ArtBench2 setting uses a subset of the styles included in ArtBench5,
  and is therefore covered by the examples shown here.
  Query validity is judged once and the same mask is shared by every method.
  }
  \label{fig:artbench-style-samples}
\end{figure*}

\begin{figure*}[h]
  \centering
  \includegraphics[width=.98\textwidth]{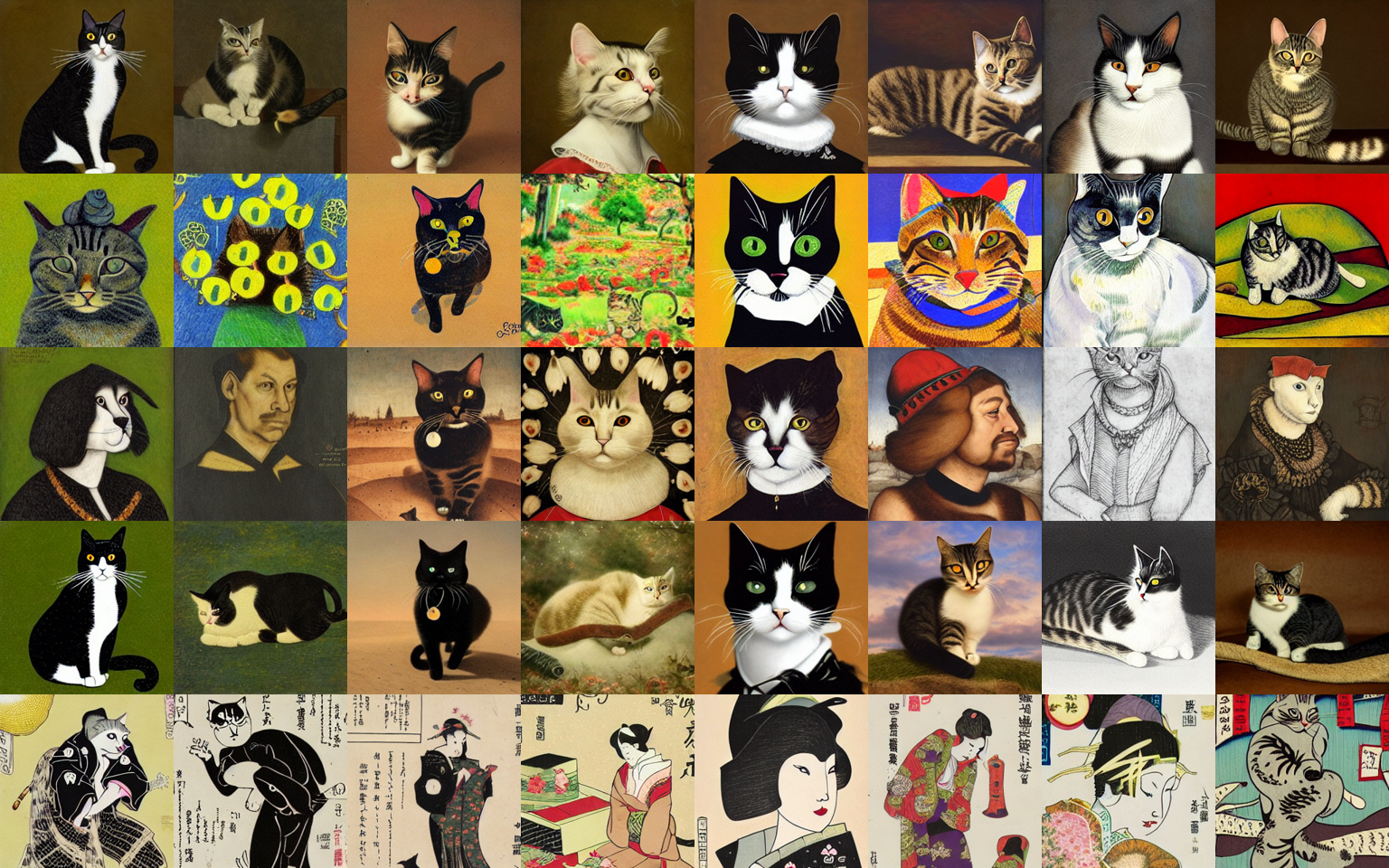}
  \caption{All 40 ArtBench5 cat--style composition generations before the shared validity mask (five styles, eight seeds each). This grid documents query coverage; only queries judged to contain both the requested object and style enter Table~\ref{tab:artbench-composition}.}
  \label{fig:artbench-composition-grid}
\end{figure*}

\begin{figure*}[h]
    \centering

    \begin{minipage}[t]{0.49\textwidth}
        \centering
        \includegraphics[
            width=\linewidth
        ]{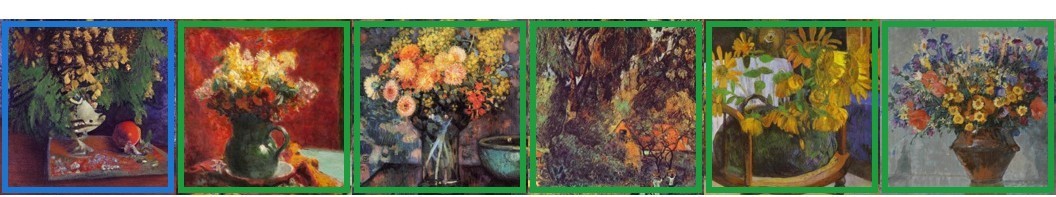}

        \vspace{-1mm}
        \footnotesize \textbf{(a)}
    \end{minipage}
    \hfill
    \begin{minipage}[t]{0.49\textwidth}
        \centering
        \includegraphics[
            width=\linewidth
        ]{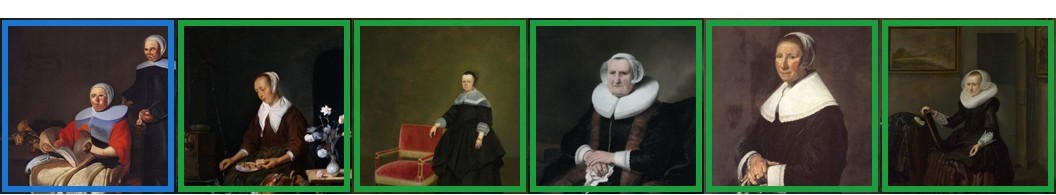}

        \vspace{-1mm}
        \footnotesize \textbf{(b)}
    \end{minipage}

    \vspace{-1mm}
    \caption{
    Representative style retrieval results produced by \method{}.
    \textbf{(a)} ArtBench2, post-impressionism.
    \textbf{(b)} ArtBench5, baroque.
    Blue denotes the generated query and green denotes style-relevant retrieved
    training examples.
    }
    \label{fig:artbench-retrieval}
\end{figure*}

\section{Data Deletion and Retraining on ArtBench2}
\label{app:artbench-deletion}

A central question in data attribution is whether the identified training
support corresponds to actual changes in the learned model when that support is
removed. We examine this question on ArtBench2 through deletion and retraining
at two levels. We first test whether removing a semantically coherent subset of
training data induces factor-specific changes in denoising behavior beyond those
caused by an equal-size random deletion. We then ask whether removing samples
identified as important by \method{} produces systematic changes in the
retraining dynamics and final model distribution. Together, these experiments
connect the response dynamics used by \method{} to observable consequences of
training-data removal.

\subsection{Style-specific changes in denoising error}

We first test whether removing coherent training support produces a
corresponding change in the model's internal denoising behavior. We use a
balanced ArtBench2 subset containing ukiyo-e and post-impressionism images and
compare two interventions: deleting all examples of one style and balanced
random deletion of the same size. Models under all conditions are retrained with
the same backbone and training protocol, so the comparison isolates the
composition of the removed data rather than the deletion magnitude.

We evaluate the retrained models on held-out queries from both styles. To
separate retraining effects from sampling variation, all model conditions are
evaluated at matched noisy query states, denoising timesteps, and noise
realizations. Differences in noise-prediction error therefore reflect changes
in the learned denoising response induced by the training-data intervention.

Let $D_c$ denote deletion of style $c$, $R$ the corresponding size-matched
random-deletion control, and $Q_c$ and $Q_{\neg c}$ the query sets from the
removed and remaining styles, respectively. We summarize the style-specific
effect using
\begin{equation}
\Delta_c =
\big[
\bar{L}^{D_c}(Q_c)-\bar{L}^{R}(Q_c)
\big]
-
\big[
\bar{L}^{D_c}(Q_{\neg c})-\bar{L}^{R}(Q_{\neg c})
\big],
\label{eq:deletion-specificity}
\end{equation}
where $\bar{L}^{A}(Q)$ denotes the average denoising MSE of model condition
$A$ on query set $Q$. A positive $\Delta_c$ indicates that removing style
$c$ affects queries from that style more strongly than queries from the
remaining style after controlling for the reduction in training-set size.

Style deletion produces a clear style-specific change in denoising behavior.
The contrasts are positive in both directions, with
$\Delta_c=1.269\times10^{-3}$ for ukiyo-e and
$\Delta_c=1.458\times10^{-3}$ for post-impressionism. Thus, removing coherent
style support changes the denoising response of corresponding queries more than
an equal-size random deletion, showing that the effect depends on which training
support is removed rather than only on how much data is removed.

Figure~\ref{fig:artbench-deletion-time} further shows that the deletion effect
is not uniform over the denoising process. Although random deletion also
perturbs prediction error, removing coherent style support produces stronger
and temporally structured changes for queries from the corresponding style.
This observation is directly relevant to \method{}: the influence of training
support is expressed through changes in internal denoising responses at
particular stages, motivating attribution based on response trajectories rather
than a single static representation.

\begin{figure}[H]
  \centering
  \includegraphics[width=0.75\textwidth]
  {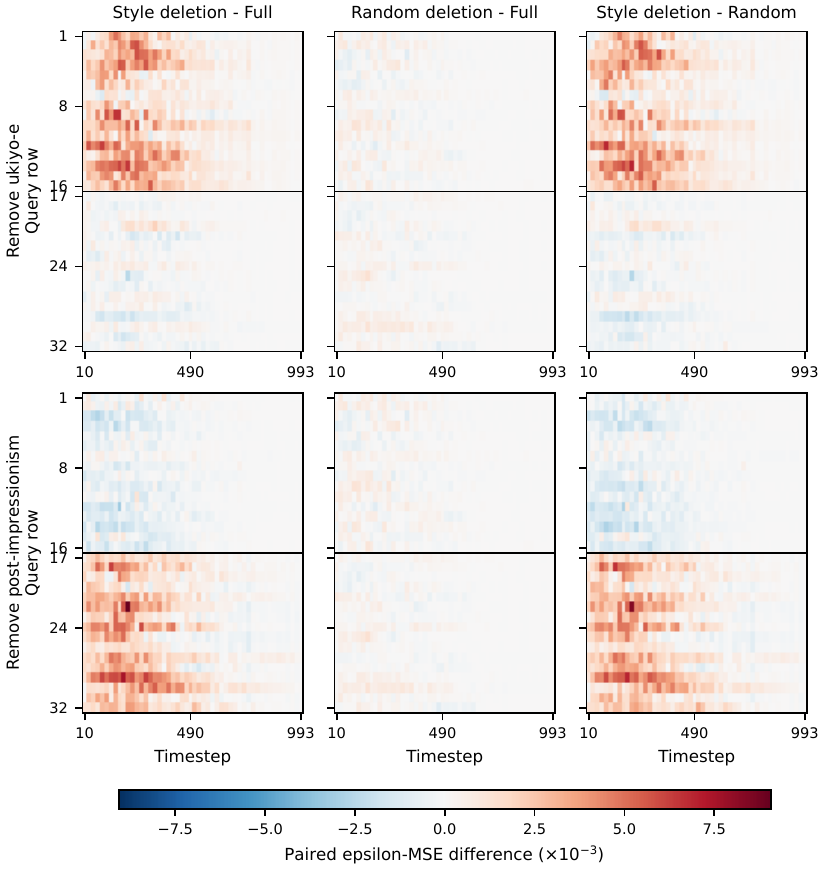}
  \caption{
  Changes in denoising error after style-specific and size-matched random
  deletion. Rows correspond to held-out queries and columns to denoising
  timesteps. The right panels show the style-minus-random contrast, revealing
  stronger and temporally structured changes for queries associated with the
  removed training style.
  }
  \label{fig:artbench-deletion-time}
\end{figure}

This experiment does not evaluate an attribution ranking directly. Instead, it
establishes that semantically coherent changes in the training set induce
corresponding and temporally structured changes in the model's denoising
behavior, providing a model-level basis for tracing training influence through
internal response dynamics.

\subsection{Changes in the distribution of retrained models}
\label{app:artbench-model-distributions}

The preceding analysis establishes that coherent training-data removal produces
structured query-level changes in denoising behavior. We next ask a stronger
attribution question: do the samples assigned high importance by \method{}
correspond to training support whose removal also changes the retrained model
itself? If the response trajectories captured by \method{} reflect meaningful
training support, removing highly attributed samples should produce a more
systematic model-level change than removing the same number of random samples.

We therefore conduct repeated retraining experiments on ArtBench2 under two
deletion strategies: (1) removing samples with the highest attribution scores
assigned by \method{}, and (2) balanced random deletion with the same deletion
ratio. The random condition controls for the generic effect of reducing the
training set, allowing differences between the two conditions to reflect the
composition of the removed samples.

Figure~\ref{fig:artbench-model-distribution} examines the resulting retraining
behavior from several complementary views. The centered loss profiles characterize
how the local training objective changes under the two deletion strategies, the
projected optimization paths show how retraining evolves from the shared
initialization, and the final distributions summarize where the independently
retrained models converge.

\begin{figure}[H]
  \centering
  \includegraphics[
    width=0.95\textwidth,
    height=0.62\textheight,
    keepaspectratio
  ]{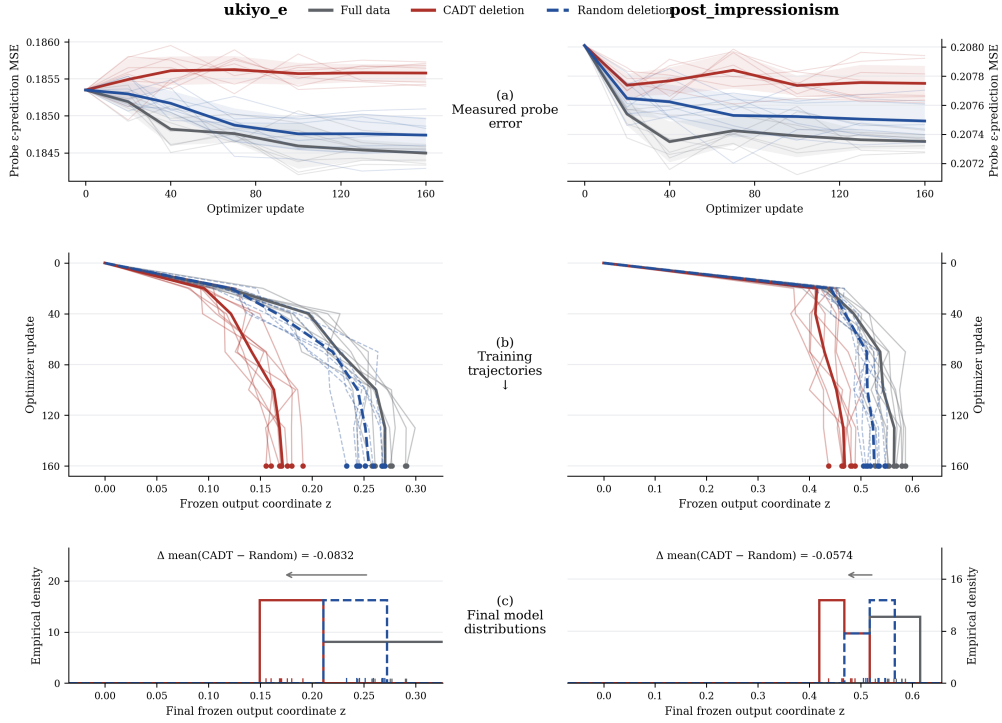}
  \caption{
  CADT deletion changes distribution of retrained ArtBench2 models.
  }
  \label{fig:artbench-model-distribution}
\end{figure}

Figure~\ref{fig:artbench-model-distribution} shows a consistent distinction
between \method{}-selected and random deletion. Removing highly attributed
samples produces a larger and more coherent displacement of the optimization
trajectory and shifts the distribution of converged models, whereas random
deletion leads to smaller and less structured changes under the same deletion
budget. This behavior is consistent with the role of the response trajectories
in \method{}: samples receiving high attribution scores provide coherent support
for the target-dependent denoising responses, so removing them produces a
systematic change in the learned model rather than only stochastic retraining
variation.

Together, the two deletion experiments provide complementary evidence for the
attribution mechanism of \method{}. Style-level deletion shows that coherent
training support produces factor-specific and temporally structured changes in
denoising behavior, while attribution-guided deletion shows that the samples
identified by \method{} also induce measurable changes in retraining dynamics
and the final model distribution. These results connect the internal response
trajectories used for attribution to actual consequences of modifying the
training data.

\section{Similar Concept-TRAK Scores with Different Denoising Structure}
\label{app:concept-scalar}

A scalar attribution score summarizes the overall contribution of a training
example, but does not explicitly preserve how that contribution is distributed
throughout denoising. We therefore ask whether training examples receiving
nearly identical Concept-TRAK scores can nevertheless exhibit different
temporal contribution patterns and internal denoising responses. This provides
a direct test of whether the dynamic information modeled by \method{} contains
structure beyond a final scalar attribution value. We study this question on
the CIFAR-100 fine class \emph{butterfly}.

We compute Concept-TRAK \citep{park2025concept} using its training-example and
concept-utility gradients. For each generated query, we select same-class
training examples with closely matched positive attribution scores. Pair
selection depends only on the final Concept-TRAK score and does not use
temporal contributions, internal responses, or image appearance. This procedure
produces 12 disjoint scalar-matched pairs.

\begin{figure}[H]
  \centering
  \includegraphics[width=0.75\textwidth]
  {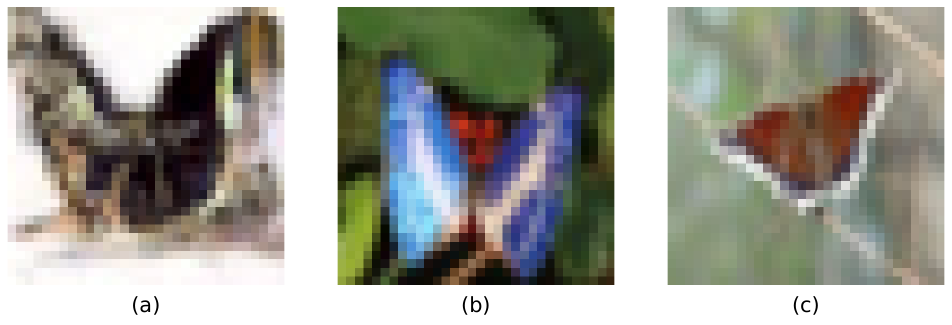}
  \caption{
  Scalar-matched examples for fine-class attribution on CIFAR-100.
  (a) A generated butterfly query.
  (b) and (c) Two butterfly training examples with nearly equal estimated
  Concept-TRAK scores. The pair is selected by score proximity under a fixed
  utility-time sample.
  }
  \label{fig:concept-scalar-pair}
\end{figure}

Figure~\ref{fig:concept-scalar-pair} shows one representative scalar-matched
pair. Although the two training examples receive nearly identical
Concept-TRAK scores for the same query, equality of the final score does not
imply that the underlying attribution evidence is accumulated in the same way.
We next examine how these similar totals are formed over denoising time.

To expose the temporal structure hidden by the scalar aggregation, we decompose
each Concept-TRAK score into query- and training-timestep contributions. Using
the same projected-gradient comparison as Concept-TRAK, the scalar score can be
written as
\begin{equation}
s_q(i)
=
\sum_{t\in\mathcal T_q}
\sum_{u\in\mathcal T_i}
C_i(t,u),
\label{eq:concept-time-decomposition}
\end{equation}
where $C_i(t,u)$ denotes the contribution associated with query timestep $t$
and training timestep $u$.

\begin{figure}[!t]
  \centering
  \includegraphics[width=0.7\textwidth]
  {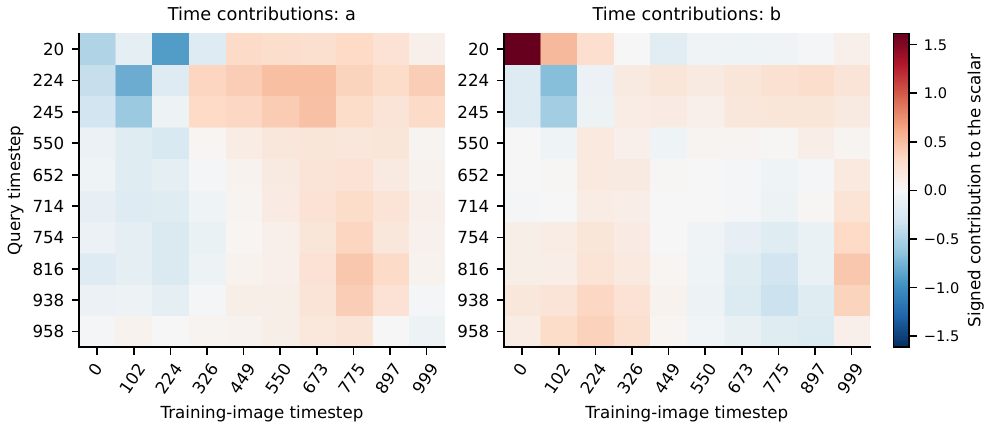}
  \caption{
  Query-time by training-time contributions to the Concept-TRAK score.
  The two maps have nearly equal total sums but exhibit different temporal
  structure.
  }
  \label{fig:concept-scalar-time}
\end{figure}

Figure~\ref{fig:concept-scalar-time} shows that nearly identical scalar scores
can arise from substantially different allocations of contribution across
query and training timesteps. Thus, the final sum preserves overall attribution
strength but discards when that evidence is expressed during denoising.

\begin{figure}[h]
  \centering
  \includegraphics[width=.6\textwidth]
  {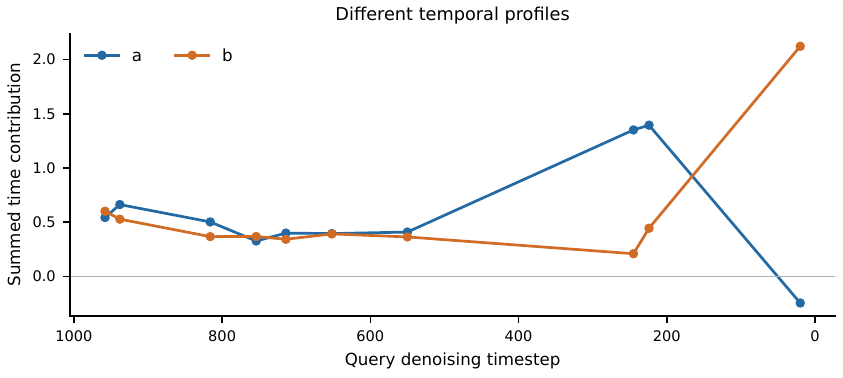}
  \caption{
  Query-time attribution profiles obtained by aggregating the contribution
  maps over training timesteps. Similar scalar scores can arise from different
  distributions of attribution across denoising.
  }
  \label{fig:concept-scalar-profile}
\end{figure}

Aggregating the contribution maps over training timesteps produces the
query-time profiles in Figure~\ref{fig:concept-scalar-profile}. The two
scalar-matched examples again exhibit different temporal profiles, making clear
that the same total score can summarize distinct patterns of attribution
support along the denoising process.

We next ask whether this temporal difference is also reflected in the actual
generation dynamics rather than only in the decomposition of the attribution
score.

Figure~\ref{fig:concept-scalar-denoising} shows that the two examples follow
different reconstruction progressions despite their similar final attribution
scores. This suggests that scalar agreement does not imply identical behavior
throughout the denoising trajectory.

\begin{figure}[h]
  \centering
  \includegraphics[width=0.7\textwidth]
  {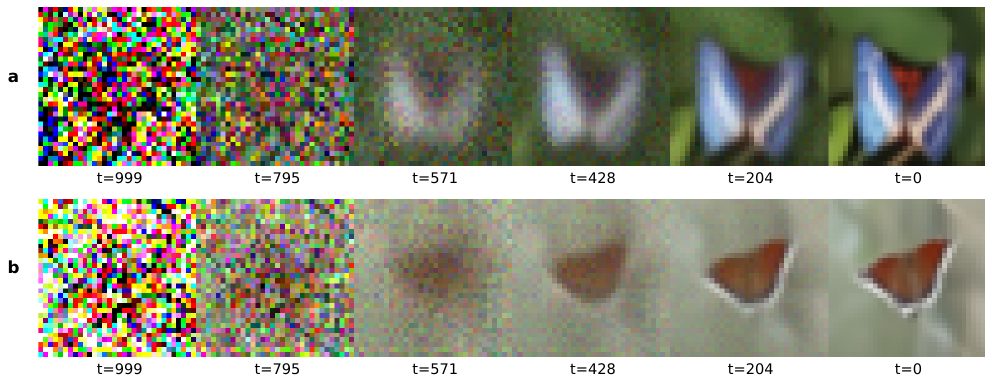}
  \caption{
  Predicted clean images along the reconstruction trajectories of two
  scalar-matched training examples. The examples exhibit different
  denoising progressions despite their similar Concept-TRAK scores.
  }
  \label{fig:concept-scalar-denoising}
\end{figure}

We further examine the model's internal responses. At each noisy state, we
contrast the fine-class and null conditions using the same input, obtaining a
local condition-specific response. Recording these responses across U-Net
blocks and timesteps yields a layer--time description of how the target concept
is expressed internally.

\begin{figure}[h]
  \centering
  \includegraphics[width=0.7\textwidth]
  {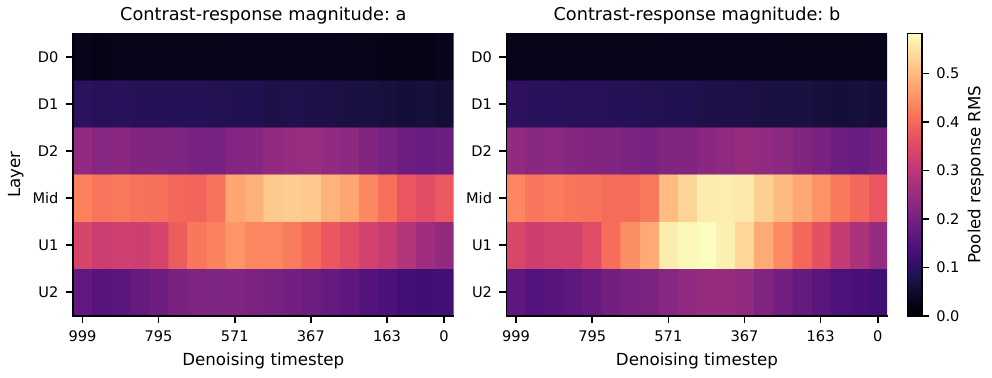}
  \caption{
  Layer--time magnitudes of the fine-class versus null-conditioned response
  for the two scalar-matched examples. Responses are evaluated at identical
  noisy states within each example.
  }
  \label{fig:concept-scalar-response}
\end{figure}

Figure~\ref{fig:concept-scalar-response} reveals distinct internal response
patterns across both network depth and denoising time. These differences are
precisely the type of information retained by \method{}: instead of collapsing
the interaction into a single value, the factor-specific responses are
organized along the denoising trajectory and their evolution is explicitly
modeled.

To verify that the observed differences are not explained only by stochastic
noise, we compare normalized response distances between scalar-matched examples
with within-example variation under different noise realizations. For the
displayed pair, the between-example distance is $0.587$, compared with $0.452$
for the within-example reference. Across all 12 matched pairs, the corresponding
median distances are $0.656$ and $0.451$, respectively. The larger
between-example distances indicate that scalar-matched examples retain
systematic differences in their internal denoising responses beyond ordinary
noise-induced variation.

Overall, these results show that similar Concept-TRAK scores can summarize
substantially different temporal contributions, reconstruction progressions,
and internal response trajectories. The scalar score remains useful as an
overall measure of attribution strength, but its aggregation removes information
about when and how the training example interacts with the denoising process.
This distinction motivates the dynamic representation in \method{}, which
retains the evolution of factor-specific responses rather than reducing them to
a single scalar interaction.

\end{document}